%% file: iclr2024_conference.tex
\documentclass{article} 
\usepackage{iclr2024_conference,times}

\input{math_commands.tex}

\usepackage{hyperref}
\usepackage{url}
\usepackage{microtype}      
\usepackage{xcolor}         
\usepackage{amsmath}
\usepackage{mathtools}
\usepackage{bbm, dsfont}
\usepackage{amsthm}
\usepackage{amssymb}
\usepackage{enumitem}

\usepackage{graphicx}
\usepackage{algorithm}
\usepackage[noend]{algpseudocode}
\usepackage{subcaption}
\newtheorem{theorem}{Theorem}[section]

\newtheorem{lemma}[theorem]{Lemma}

\newtheorem{assumption}[theorem]{Assumption}

\usepackage[textsize=tiny]{todonotes}

\renewcommand{\epsilon}{\varepsilon}
\newcommand{\ep}{\varepsilon}

\definecolor{ed}{RGB}{225,0,100}

\title{PAC Prediction Sets Under Label Shift}

\usepackage{authblk}

\author[1]{Wenwen Si}
\author[2]{Sangdon Park}
\author[1]{Insup Lee}
\author[3]{Edgar Dobriban}
\author[1]{Osbert Bastani \footnote{Author e-mail addresses: \{wenwens, lee, obastani\}@seas.upenn.edu, sangdon@postech.ac.kr, dobriban@wharton.upenn.edu} .}
\affil[1]{Department of Computer \& Information Science, University of Pennsylvania}
\affil[2]{Graduate School of AI, POSTECH}
\affil[3]{Department of Statistics and Data Science, University of Pennsylvania}

\begin{document}

\maketitle

\begin{abstract}
Prediction sets capture uncertainty by predicting sets of labels rather than individual labels, enabling downstream decisions to conservatively account for all plausible outcomes. Conformal inference algorithms construct prediction sets guaranteed to contain the true label with high probability. These guarantees fail to hold in the face of distribution shift, which is precisely when reliable uncertainty quantification can be most useful. We propose a novel algorithm for constructing prediction sets with PAC guarantees in the label shift setting. 
This method estimates the predicted probabilities of the classes in a target domain, as well as the confusion matrix, 
then propagates uncertainty in these estimates through a Gaussian elimination algorithm to compute confidence intervals for importance weights.
Finally, it uses these intervals to construct prediction sets.
We evaluate our approach on five datasets:
the CIFAR-10, ChestX-Ray and Entity-13 image datasets, 
the tabular CDC Heart Dataset, 
and the AGNews text dataset. Our algorithm satisfies the PAC guarantee while producing smaller, more informative, prediction sets compared to several baselines.
\end{abstract}

\section{Introduction}

Uncertainty quantification can be a critical tool for building reliable systems from machine learning components. For example, a medical decision support system can convey uncertainty to a doctor, or a robot can act conservatively with respect to uncertainty. These approaches are particularly important when the data distribution shifts as the predictive system is deployed, since they enable the decision-maker to react to degraded performance.

Conformal prediction~\citep{vovk2005algorithmic,angelopoulos2021gentle} is a promising approach to uncertainty quantification,
converting machine learning models into \emph{prediction sets}, which output sets of labels instead of a single label. Under standard assumptions (i.i.d. or exchangeable data), it guarantees that the prediction set contains the true label with high probability. We consider \emph{probably approximately correct (PAC)} (or \emph{training-conditional}) guarantees~\citep{vovk2012conditional,park2019pac}, which ensure high probability coverage over calibration datasets used to construct the prediction sets.

In this paper, we propose a novel prediction set algorithm that provides PAC guarantees under the \emph{label shift} setting, 
where the distribution of the labels may shift, but the distribution of covariates conditioned on the labels remains fixed. For instance, during a pandemic, a disease may spread to a much larger fraction of the population, but the 
manifestations of the disease may remain the same. As another example, real-world data may have imbalanced classes, unlike the balanced classes typical of curated training datasets. We consider the unsupervised domain adaptation setting~\citep{ben2006analysis}, where we are given labeled examples from a \emph{source domain}, but only unlabeled examples from the \emph{target domain}, and care about performance in the target domain. 

A standard way to adapt conformal inference to handle distribution shift is by using importance weighting to ``convert'' data from the source distribution into data from the target distribution~\citep{tibshirani2019conformal,park2021pac}. 
In the label shift setting, one can express the importance weights as $w^*=\mathbf{C}_P^{-1}q^*$, where $\mathbf{C}_P$ is the \emph{confusion matrix} and $q^*$ is the \emph{distribution of predicted labels}~\citep{lipton2018detecting}; see details below. 
However, since $\mathbf{C}_P$ and $q^*$ are unknown and must be estimated based on a finite dataset,
this introduces additional errors that are not accounted for by existing weighting methods.

To address this problem, we construct confidence intervals around $\mathbf{C}_P$ and $q^*$, and then devise a novel algorithm to propagate these intervals through a Gaussian elimination algorithm used to compute $w^*$. 
We can then leverage an existing strategy for constructing PAC prediction sets when given confidence intervals for the importance weights~\citep{park2021pac}.

\begin{figure}[t]
\centering
\includegraphics[width=0.98\linewidth]{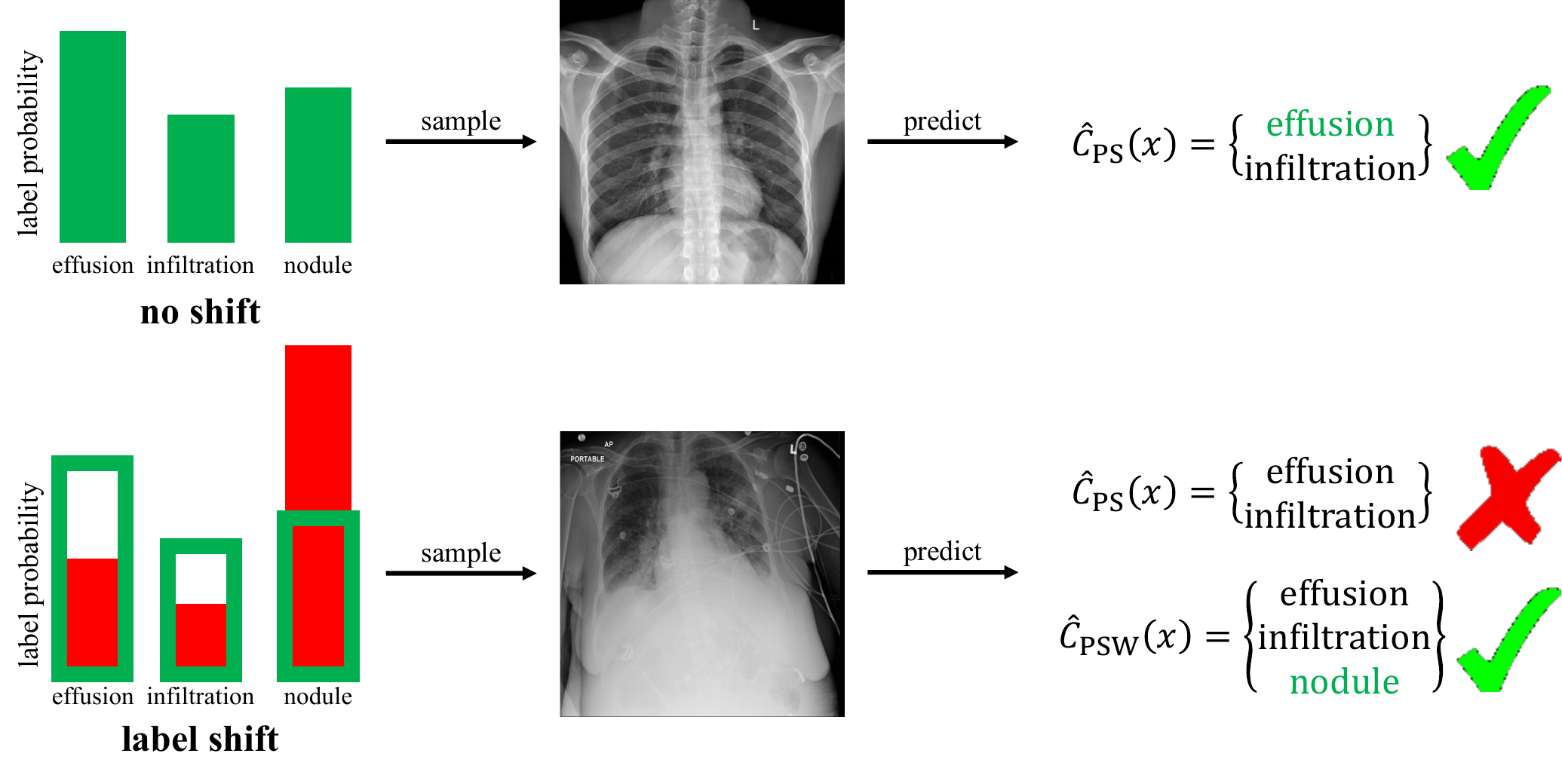}
\caption{An example of our approach on the ChestX-ray dataset. In the unshifted setting, standard PAC prediction sets guarantee high-probability coverage, but this guarantee fails under label shift. Our approach addresses this challenge and remains valid in the shifted environment.}
\vspace{-2em}
\label{fig:example}
\end{figure}

We empirically evaluate our approach on five datasets across three application domains: CIFAR-10~\citep{krizhevsky2009learning} and Entity-13~\citep{santurkar2020breeds} in the computer vision domain, the CDC Heart Dataset~\citep{cdcheart} and ChestX-ray~\citep{national2022nih} in the medical domain, and AGNews~\citep{zhang2015character} in the language domain. 

\textbf{Contributions.}
We propose a novel algorithm for constructing PAC prediction sets in the presence of label shift, which computes provably valid intervals around the true importance weights. 
Our algorithm is based on a novel technique for propagating confidence intervals through the updates of Gaussian elimination, which to the best of our knowledge is a novel approach to uncertainty propagation in a prediction set construction setting. This idea may be of independent interest and applicable to other linear algebraic computations.
Finally, empirically demonstrate that our approach satisfies the PAC guarantee while constructing smaller prediction sets than several baselines. 

\textbf{Example.} Figure~\ref{fig:example} illustrates an application of our technique to the ChestX-ray dataset. In medical settings, PAC prediction sets (denoted PS) provide a rigorous way to quantify uncertainty for making downstream decisions; in particular, they guarantee that the prediction set contains the true label (in this case, a diagnosis) with high probability. However, label shift happens commonly in medical settings, for instance, many illnesses have varying rates of incidence over time even when the patient population remains the same. Unfortunately, label shift can invalidate the PAC coverage guarantee. Our approach (denoted PSW) corrects for the label shift via importance weighting; it does so in a provably correct way by propagating uncertainty through the Gaussian elimination computation used to construct importance weights. The resulting prediction sets satisfy the PAC guarantee.

\textbf{Related work.} There has been recent interest in conformal inference under distribution shift, much of it focusing on covariate shift~\citep{tibshirani2019conformal, lei2021conformal,qiu2022distribution}. \citep{podkopaev2021distribution} develop methods for marginal coverage under label shift, 
whereas we are interested in training-set conditional---or PAC---guarantees. Furthermore, they assume that the true importance weights are known, which is rarely the case.
In the label shift setting, the importance weights can be estimated~\citep{lipton2018detecting}, but as we show in our experiments, uncertainty in these estimates must be handled for the PAC guarantee to hold. 

We leverage the method of~\citep{park2021pac} to handle estimation error in the importance weights.
That work studies covariate shift, and uses a heuristic to obtain intervals around the importance weights. 
For the label shift setting, we can in fact obtain stronger guarantees: we modify Gaussian elimination to propagate uncertainty through the computation of the weights $w^*=\mathbf{C}_P^{-1}q^*$.
We give a more comprehensive discussion of related work in Appendix~\ref{sec:related}. Our code is  available at \url{https://github.com/averysi224/pac-ps-label-shift} for reproduction of our experiments.

\section{Problem Formulation}


\subsection{Background on Label Shift}\label{sec:labelshift}

Consider the goal of training a classifier $g:\mathcal{X}\to\mathcal{Y}$, where $\mathcal{X}\subseteq\mathbb{R}^d$ is the covariate space, and $\mathcal{Y}=[K]=\{1,...,K\}$ is the set of labels. We consider the setting where we train on one distribution $P$ over $\mathcal{X}\times\mathcal{Y}$---called the \emph{source}---with a probability density function (PDF) $p:(x,y)\mapsto p(x,y)$, and evaluate on a potentially different test distribution $Q$---called the \emph{target}---with PDF $q:(x,y)\mapsto q(x,y)$. 
We focus on the unsupervised domain adaptation setting~\citep{ben2007analysis}, where we are given an i.i.d.~sample $S_m\sim P^m$ of $m$ labeled datapoints, and  an i.i.d.~sample of $n$ unlabeled datapoints $T_X^n\sim Q_X^n$. The label shift setting~\citep{lipton2018detecting} assumes that only the label distribution $Q_Y$ may change from $P_Y$, and the conditional covariate distributions remain the same:
\begin{assumption}
\label{assump:a1}
\rm
(Label shift) We have $p(x\mid y) = q(x\mid y)$ for all $x \in \mathcal{X}, y \in \mathcal{Y}$.
\end{assumption} 
We denote $p(y) =P_Y(Y=y)$ for all $y\in \mathcal{Y}$ and analogously for $Q$.
\citep{lipton2018detecting} consider two additional mild assumptions:
\begin{assumption}
\label{assump:a2}
\rm
For all $y \in \mathcal{Y}$ such that $q(y) > 0$, we have $p(y) > 0$.
\end{assumption}
Next, given the trained classifier $g: \mathcal{X} \rightarrow \mathcal{Y}$
let $\mathbf{C}_P \in \mathbb{R}^{K\times K}$ 
denote its expected confusion matrix---i.e., $c_{ij}:=(\mathbf{C}_P)_{ij} = \mathbb{P}_{(X,Y)\sim P}(g(X)=i, Y=j)$.
\begin{assumption}
\label{assump:a3}
The confusion matrix $\mathbf{C}_P$ is invertible.
\end{assumption}
This last assumption requires that the per-class expected predictor outputs be linearly independent; for instance, it is satisfied when $g$ is reasonably accurate across all labels. In addition, one may test whether this assumption holds~\citep{lipton2018detecting}.

Denoting the importance weights $w^* \coloneqq (q(y)/p(y))_{y\in\mathcal{Y}}\in\mathbb{R}^K$, 
and
$\hat{y}:=g(x)$,
we will write $p(\hat y|y)=\mathbb{P}_{(X,Y)\sim P_X}[g(X)=\hat y|Y=y]$, and define 
$p(\hat y,y)$, $p_{\hat y}$ as well as the corresponding expressions for $q$ analogously.
Since $\hat y$
depends only on $x$, we have
$q(\hat{y}\mid y) = p(\hat{y}\mid y)$. 
Thus, see e.g., \cite{lipton2018detecting},
\begin{align*}
q_{\hat{y}} 
=\sum_{y\in \mathcal{Y}} q(\hat{y}\mid y)q(y)
=\sum_{y\in\mathcal{Y}} p(\hat{y}\mid y)q(y)
=\sum_{y\in\mathcal{Y}} p(\hat{y}, y) \frac{q(y)}{p(y)},
\end{align*}
or in a matrix form,
$q^* = \mathbf{C}_Pw^*$,
where
$q^* \coloneqq (q_{\hat{y}} )_{\hat{y}\in\mathcal{Y}}\in\mathbb{R}^K$. As we assume $\mathbf{C}_P$ is invertible,
\begin{align}
\label{eq:confq}
w^*= \mathbf{C}_P^{-1} q^*.
\end{align}
Our algorithm uses this equation to approximate $w^*$, and then use its approximation to construct PAC prediction sets that remain valid under label shift.

\subsection{PAC Prediction Sets Under Label Shift}\label{sec:pacps}

We are interested in constructing a \emph{prediction set} $C:\mathcal{X}\to2^{\mathcal{Y}}$, which outputs a set of labels $C(x)\subseteq\mathcal{Y}$ for each given input $x\in\mathcal{X}$ rather than a single label. The benefit of outputting a set of labels is that we can obtain correctness guarantees such as:
\begin{align}
\label{eqn:ac}
\mathbb{P}_{(X,Y)\sim P}[Y\in C(X)]\ge1-\epsilon,
\end{align}
where $\epsilon\in(0,1)$ is a user-provided confidence level. Then, downstream decisions can be made in a way that accounts for all labels $y\in C(x)$ rather than for a single label. Thus, prediction sets quantify uncertainty.
Intuitively, \eqref{eqn:ac} can be achieved if we output $C(x)=\mathcal{Y}$ for all $x\in\mathcal{X}$, but this is not informative. Instead, the typical goal is to output prediction sets that are as small as possible. 

The typical strategy for constructing prediction sets is to leverage a fixed existing model. In particular, we assume given a \emph{scoring function} $f:\mathcal{X}\times\mathcal{Y}\to\mathbb{R}$; most deep learning algorithms provide such scores in the form of predicted probabilities, with the corresponding classifier being $g(x)=\operatorname*{\arg\max}_{y\in\mathcal{Y}}f(x,y)$. The scores do not need to be reliable in any way; if they are unreliable, the PAC prediction set algorithm will output larger sets. Then, we consider prediction sets parameterized by a  real-valued threshold $\tau\in\mathbb{R}$:
\begin{align*}
C_\tau(x)=\{y\in\mathcal{Y}\mid f(x,y)\ge\tau\}.
\end{align*}
In other words, we include all labels with score at least $\tau$.
First, we focus on correctness for $P$, in which case we only need $S_m$, usually referred to as the calibration set. 
Then, a prediction set algorithm constructs a threshold $\hat\tau(S_m)\in\mathbb{R}$ and returns $C_{\hat\tau(S_m)}$.



Finally, we want $\hat\tau$ to satisfy (\ref{eqn:ac}); one caveat is that it may fail to do so due to randomness in $S_m$.
Thus, we allow an additional probability $\delta\in\mathbb{R}$ of failure, resulting in the following desired guarantee:
\begin{align}
\label{eqn:pac}
\mathbb{P}_{S_m\sim P^m} [\mathbb{P}_{(X,Y)\sim P} [Y \in C_{\hat\tau(S_m)}(X)] \geq 1-\epsilon] \geq 1-\delta.
\end{align}
\citet{vovk2012conditional,park2019pac} propose an algorithm $\hat\tau$ that satisfies (\ref{eqn:pac}),
see Appendix~\ref{sec:k}.

Finally, we are interested in constructing PAC prediction sets in the label shift setting, using both the labeled calibration dataset $S_m\sim P^m$ from the source domain, 
and the unlabeled calibration dataset $T_n^X\sim Q^n$ from the target distribution. 
Our goal is to construct $\hat\tau(S_m,T_n^X)$ based on both $S_m$ and $T_n^X$, which satisfies the coverage guarantee over $Q$ instead of $P$:
\begin{align}
\label{eqn:labelpac}
\mathbb{P}_{S_m\sim P^m, T^X_n\sim Q^n_X} \left[\mathbb{P}_{(X,Y)\sim Q}[Y\in C_{\hat\tau(S_m,T_n^X)}(X)]\ge1-\epsilon\right]\geq 1-\delta.
\end{align}
Importantly, the inner probability is over the shifted distribution $Q$ instead of $P$.

\section{Algorithm}


To construct prediction sets valid under label shift,
we first notice that it is enough to find element-wise confidence intervals for the importance weights $w^*$.
Suppose that we can construct $W=\prod_{k\in \mathcal{Y}}[\underline{w}_k,\overline{w}_k]\subseteq\mathbb{R}^K$ 
such that $w^*\in W$.
Then, when adapted to our setting,
the results of \citet{park2021pac}---originally for the covariate shift problem---provide an algorithm 
that returns a threshold  $\hat\tau(S_m,V,W,b)$,
where $V\sim\text{Uniform}([0,1])^K$ is a vector of random variables,
such that 
\begin{align}\label{thm:ipw}
\mathbb{P}_{S_m\sim P^m,V\sim U^K}\left[\mathbb{P}_{(X,Y)\sim Q}[Y\in C_{\hat\tau(S_m,V,W,b)}]\ge1-\epsilon\right]\ge1-\delta.
\end{align}

This is similar to \eqref{eqn:labelpac} but one minor point is that it accounts 
for the randomness used by our algorithm---via $V$---in the outer probability.
We give details in Appendix~\ref{psdsalg}. 

The key challenge is to construct 
the elementwise confidence interval
$W=\prod_{k\in\mathcal{Y}}[\underline{w}_k,\overline{w}_k]$ such that $w^*\in W$ with high probability.
The approach from \cite{park2021pac} for the covariate shift problem relies on training a source-target discriminator, 
which is not possible in our case since we do not have class labels from the target domain. 
Furthermore, \cite{park2021pac} do not provide conditions under which 
one can provide a valid confidence interval for the importance weights.

Our algorithm uses a novel approach, where we propagate intervals through the computation of importance weights.
The weights $w^*$ are determined 
by the system of linear equations
$\mathbf{C}_Pw^*=q^*$.
Since $\mathbf{C}_P$ and $q^*$ are unknown, we start by constructing \emph{element-wise} confidence intervals
\begin{align}
\label{eqn:inputbounds}
\underline{\mathbf{C}}_P\le\mathbf{C}_P\le\overline{\mathbf{C}}_P
\qquad\text{and}\qquad
\underline{q}^*\le q^*\le\overline{q}^*,
\end{align}
with a probability of at least $1-\delta$ over our calibration datasets $S_m$ and $T_n^X$. 
We then propagate these confidence intervals through each step of Gaussian elimination, 
such that at the end of the algorithm, we obtain confidence intervals for its output---i.e.,
\begin{align}
\label{eqn:outputbounds}
\underline{w}^*\le w^*\le\overline{w}^*
\qquad\text{with probability at least}~1-\delta.
\end{align}
Finally, we can use (\ref{eqn:outputbounds}) with the algorithm from \citep{park2021pac} to construct PAC prediction sets under label shift. We describe our approach below.

\subsection{Elementwise Confidence Intervals for \texorpdfstring{$\mathbf{C}_P$}{CP} and \texorpdfstring{$q^*$}{q}}

Recall that $\mathbf{C}_P=(c_{ij})_{ij\in\mathcal{Y}}$ and $q^*=(q_{\hat y})_{{\hat y}\in\mathcal{Y}}$. 
Note that $c_{ij}=\mathbb{P}[g(X)=i,Y=j]$ is the mean of the Bernoulli random variable $\mathbbm{1}(g(X)=i, Y=j)$ over the randomness in $(X,Y)\sim P$. 
Similarly, $q_k$ is the mean of $\mathbbm{1}(g(X)=k)$ over the randomness in $X\sim Q_X$. 
Thus, we can use the Clopper-Pearson (CP) intervals~\citep{clopper1934use}
for a Binomial success parameter 
 to construct intervals around $c_{ij}$ and $q_k$. 
Given  a confidence level $\delta\in(0,1)$ and the
sample mean $\hat{c}_{ij}=\frac{1}{m}\sum_{(x,y)\in S_m}\mathbbm{1}(g(x)=i, y=j)$---distributed as a scaled Binomial random variable---this is an interval $\text{CP}(\hat{c}_{ij},m,\delta)=[\underline{c}_{ij},\overline{c}_{ij}]$ such that
\begin{align*}
\mathbb{P}_{S_m\sim P^m}[c_{ij}\in\text{CP}(\hat{c}_{ij},m,\delta)]\ge1-\delta.
\end{align*}
Similarly, for $q_k$, we can construct CP intervals based on $\hat{q}_k=\frac{1}{n}\sum_{x\in T_n^X}\mathbbm{1}(g(x)=k)$. Together, for confidence levels $\delta_{ij}$ and $\delta_k$ chosen later, we obtain for all $i,j,k\in[K]$,
\begin{align*}
\mathbb{P}_{S_m\sim P^m}\left[ \underline{c}_{ij} \le c_{ij} \le \overline{c}_{ij} \right] &\ge 1 - \delta_{ij},
\qquad
\mathbb{P}_{T_n^X\sim Q_X^n}\left[\underline{q}_k \le q_k \le \overline{q}_k\right] \ge 1 - \delta_k.
\end{align*}
Then, the following result holds by a union bound:
Given any $\delta_{ij},\delta_k\in(0,\infty)$, for all $i,j,k\in[K]$, letting $[\underline{c}_{ij},\overline{c}_{ij}]=\text{CP}(\hat{c}_{ij},m,\delta_{ij})$ and $[\underline{q}_k,\overline{q}_k]=\text{CP}(\hat{q}_k,n,\delta_k)$, and letting $\delta=\sum_{i,j\in[K]}\delta_{ij}+\sum_{k\in[K]}\delta_k$, we have
\begin{align}\label{lem:cp}
\mathbb{P}_{S_m\sim P^m,T_n^X\sim Q_X^n}\left[ \bigwedge_{i,j\in[K]}\underline{c}_{ij} \le c_{ij} \le \overline{c}_{ij}  , \bigwedge_{k\in[K]} \underline{q}_k \le q_k \le \overline{q}_k\right] \ge 1 - \delta.
\end{align}



\subsection{Gaussian Elimination with Intervals}


We also need to set up notation for Gaussian elimination, which requires us to briefly recall the algorithm. 
 To solve $\mathbf{C}_Pw^*=q^*$, Gaussian elimination \citep[see e.g.,][]{golub2013matrix} proceeds in two phases. 
Starting with $c^0=\mathbf{C}_P$ and $q^0=q^*$,
 on iteration $t\ge 1$, Gaussian elimination uses row $k=t$ to eliminate the $k$th column of rows $i\in\{k+1,...,K\}$ (we introduce a separate variable $k$ for clarity). In particular, 
if $c_{kk}^t\neq 0$,
we denote
\begin{align*}
c_{ij}^{t+1} &=
\begin{cases}
c_{ij}^t-\dfrac{c_{ik}^tc_{kj}^t}{c_{kk}^t} &\text{if } i>k, \\
c_{ij}^t &\text{otherwise;}
\end{cases}
\qquad\,
q_i^{t+1} = \begin{cases}
q_i^t - \dfrac{c_{ik}^tq_k^t}{c_{kk}^t} &\text{if }i > k \\
q_i^t &\text{otherwise,}
\end{cases}
\qquad\forall i,j\in[K].
\end{align*}
If $c_{kk}^t= 0$, but there is an element $j>k$ in the $k$th column such that $c_{jk}^t\neq 0$, the $k$th and the $j$th rows are swapped and the above steps are executed. If no such element exists, the algorithm proceeds to the next step. At the end of the first phase, the matrix $c^{K-1}$ has all elements below the diagonal equal to zero---i.e., $c^{K-1}_{ij}=0$ if $j<i$. 
In the second phase, the Gaussian elimination algorithm solves for $w_i^*$ backwards from $i=K$ to $i=1$, introducing the following notation.
For each $i$, if $c_{ii}^{K-1}\neq 0$, we denote\footnote{The algorithm requires further discussion if $c_{ii}^{K-1}=0$ \citep{golub2013matrix}; this does not commonly happen in our motivating application so we will not consider this case. See Appendix~\ref{sec:positivity} for details.}
$w_i^*=(q_i-s_i)/c_{ii}^{K-1}$,
 where $s_i=\sum_{j=i+1}^Kc_{ij}^{K-1}w_j^*$.

In our setting, we do not know $c^0$ and $q^0$; instead, we assume given entrywise confidence intervals as in \eqref{eqn:inputbounds}, which amount to
$\underline{c}^0\le c^0\le\overline{c}^0$ and $\underline{q}^0\le q^0\le\overline{q}^0$. 
We now work on the event $\Omega$ that these bounds hold, 
and prove that our algorithm works on this event; later, we combine this result with Equation~\ref{lem:cp} to obtain a high-probability guarantee. Then, our goal is to compute $\underline{c}^t,\overline{c}^t,\underline{q}^t,\overline{q}^t$ such that
for all iterations $t\in\{0,1,...,K-1\}$, 
we have elementwise confidence intervals
specified by
$\underline{c}^t,\overline{c}^t$,
$\underline{q}^t$ and $\overline{c}^t$
for the outputs 
$c^t,q^t$
of the Gaussian elimination algorithm:
\begin{align}
\label{eqn:invariant}
\underline{c}^t\le c^t\le\overline{c}^t
\qquad\text{and}\qquad
\underline{q}^t\le q^t\le\overline{q}^t.
\end{align}
The base case $t=0$ holds by the assumption.
Next, to propagate the uncertainty through the Gaussian elimination updates
for each iteration $t\in[K-1]$, our algorithm sets
\begin{align}
\label{eqn:clower}
\underline{c}_{ij}^{t+1}&=\begin{cases}
0 &\text{if } i > k,~j \le k, \\
    \underline{c}_{ij}^t - \dfrac{\overline{c}_{ik}^t\overline{c}_{kj}^t}{\underline{c}_{kk}^t} &\text{if } i,j > k, \\
\underline{c}_{ij}^t &\text{otherwise}
\end{cases}
\qquad\forall i,j\in[K]
\end{align}
for the lower bound, and computes
\begin{align}
\label{eqn:cupper}
\overline{c}_{ij}^{t+1}&=\begin{cases}  
0 &\text{if } i > k,~j \le k, \\
\overline{c}_{ij}^t - \dfrac{\underline{c}_{ik}^t\underline{c}_{kj}^t}{\overline{c}_{kk}^t} &\text{if } i,j > k, \\
\overline{c}_{ij}^t &\text{otherwise}
\end{cases}
\qquad\forall i,j\in[K]
\end{align}
for the upper bound. The first case handles the fact that Gaussian elimination is guaranteed to zero out entries below the diagonal, and thus these entries have no uncertainty  remaining.
The second rule 
constructs confidence intervals based on the previous intervals and the algebraic update formulas used in Gaussian elimination for the entries for which $i,j>k$.
For instance, the above confidence intervals use that on the event $\Omega$, and by induction on $t$,
if 
$\underline{c}_{ij}^t\ge0$ and $\underline{c}_{ii}^t>0$ for all $i,j\in[K]$ and for all $t$,
the Gaussian elimination update $c_{ij}^{t+1} = c_{ij}^t-c_{it}^tc_{tj}^t/c_{tt}^t$ can be upper bounded as 
\begin{equation}\label{ind}
c_{ij}^{t+1} = c_{ij}^t-\dfrac{c_{it}^tc_{tj}^t}{c_{tt}^t}
\le 
\overline{c}_{ij}^t - \dfrac{\underline{c}_{it}^t\underline{c}_{tj}^t}{\overline{c}_{tt}^t} 
= \overline{c}_{ij}^{t+1},
\end{equation}
The assumptions that $\underline{c}_{ij}^t\ge0$ and $\underline{c}_{ii}^t>0$ for all $i,j\in[K]$ and for all $t$
may appear a little stringent,
but the former can be removed at the cost of slightly larger intervals propagated to the next step, see Section~\ref{sec:positivity}.
The latter condition is satisfied by any classifier that obtains sufficient accuracy on all labels. 
We further discuss these conditions in Section~\ref{sec:positivity}. 
The third rule in \eqref{eqn:clower} and \eqref{eqn:cupper} handles the remaining entries,
which do not change; and thus the confidence intervals from the previous step can be used. The rules for $q$ are similar, and have a similar justification:
\begin{align}\label{eqn:q}
\underline{q}_i^{t+1}=\begin{cases}
\underline{q}_i^t-\dfrac{\overline{c}_{ik}^t\overline{q}_i^t}{\underline{c}_{kk}^t} &\text{if }i>k, \\
\underline{q}_i^t &\text{otherwise;}
\end{cases}
\qquad
\overline{q}_i^{t+1}=\begin{cases}
\overline{q}_i^t-\dfrac{\underline{c}_{ik}^t\underline{q}_i^t}{\overline{c}_{kk}^t} &\text{if }i>k, \\
\overline{q}_i^t &\text{otherwise.}
\end{cases}
\qquad\forall i\in[K].
\end{align}
For these rules, our algorithm assumes $\underline{q}_i^t\ge0$ for all $i\in[K]$ and all $t$, and raises an error if this fails. As with the first condition above, this one can be straightforwardly relaxed; see Appendix~\ref{sec:positivity}.

In the second phase, we compute $w_i^*$ starting from $i=K$ and iterating to $i=1$. On iteration $i$, we assume we have the confidence intervals $\underline{w}_j^*\le w_j^*\le\overline{w}_j^*$ for $j>i$. Then, we compute confidence intervals for the sum $s_i$, 
which again have a similar justification based on the Gaussian elimination updates:
\begin{align}
\label{eqn:slowerupper}
\underline{s}_i=\sum_{j=i+1}^n\underline{c}_{ij}^{K-1}\underline{w}_j^*
\qquad\text{and}\qquad
\overline{s}_i=\sum_{j=i+1}^n\overline{c}_{ij}^{K-1}\overline{w}_j^*,
\end{align}
and show that they satisfy $\underline{s}_i\le s_i\le\overline{s}_i$ on the event $\Omega$. 
Finally, we compute confidence intervals for $w_i^*$, assuming $\underline{c}_{ii}^{K-1}>0$:
\begin{align}
\label{eqn:wlowerupper}
\underline{w}_i^*=\frac{\underline{q}_i-\overline{s}_i}{\overline{c}_{ii}^{K-1}}
\qquad\text{and}\qquad
\overline{w}_i^*=\frac{\overline{q}_i-\underline{s}_i}{\underline{c}_{ii}^{K-1}},
\end{align}
for which we can show that they satisfy $\underline{w}_i^*\le w_i^*\le\overline{w}_i^*$
based on the Gaussian elimination updates. 
Letting $W=\{w\mid\underline{w}^*\le w\le\overline{w}^*\}$, we have the following (see Appendix~\ref{sec:gaussianelimproof} for a proof).
\begin{lemma}[Elementwise Confidence Interval for Importance Weights]
\label{lem:gaussianelim}
If (\ref{eqn:inputbounds}) holds,
and for all $ i,j,t\in[K]$, $\underline{c}_{ij}^t\ge0$, $\underline{c}_{ii}^t>0$, and $\underline{q}_i^t\ge0$, 
then $w^*=\mathbf{C}_P^{-1}q^*\in W$.
\end{lemma}
We mention here that the idea of 
algorithmic uncertainty propagation 
may be of independent interest.
In future work,
it may further be developed to other 
methods for solving linear systems (e.g., the LU decomposition, \cite{golub2013matrix}), 
and other linear algebraic and numerical computations.

\subsection{Overall Algorithm}

\begin{algorithm*}[t]
\small
\caption{PS-W: PAC prediction sets in the label shift setting.}
\label{alg:main}
\begin{algorithmic}[1]
\Procedure{LabelShiftPredictionSet}{$S_m, T^X_n, f, \epsilon, \delta$}
\State $\underline{c},\overline{c},\underline{q},\overline{q}\gets$ \Call{CPInterval}{$S_m, T^X_n, x\mapsto\operatorname*{\arg\max}_{y\in\mathcal{Y}}f(x,y), \frac{K(K+1)}{(K(K+1)+1)}\delta$}
\State $W \gets$ \Call{IntervalGaussianElim}{$\underline{c},\overline{c},\underline{q},\overline{q}$}
\If{$W=\varnothing$} \Return $\varnothing$ 
\EndIf
\State \Return \Call{IWPredictionSet}{$S_m, f, W, \epsilon, \delta/[K(K+1)+1]$}
\EndProcedure
\Procedure{CPInterval}{$S_m, T^X_n, g, \delta$}
\For{$i,j\in[K]$}
\State Compute $[\underline{c}_{ij},\overline{c}_{ij}]=\text{CP}\left(m^{-1}\sum_{(x,y)\in S_m}\mathbbm{1}(g(x)=i,y=j),m,\delta/(K(K+1))\right)$
\EndFor
\For{$k\in[K]$}
\State Compute $[\underline{q}_k,\overline{q}_k]=\text{CP}\left(n^{-1}\sum_{x\in T_n^X}\mathbbm{1}(g(x)=k),n,\delta/(K(K+1))\right)$
\EndFor
\State \Return $\underline{c},\overline{c},\underline{q},\overline{q}$
\EndProcedure
\Procedure{IntervalGaussianElim}{$\underline{c}^0,\overline{c}^0,\underline{q}^0,\overline{q}^0$}
\For{$t\in[1,...,K-1]$}
\For{$i,j\in[K]$}
\State Compute $\underline{c}_{ij}^t,\overline{c}_{ij}^t$ using (\ref{eqn:clower}) \& (\ref{eqn:cupper}), and $\underline{q}_i^t,\overline{q}_i^t$ using (\ref{eqn:q})
\If{$\underline{c}_{ij}^t<0$ for some $i\neq j$ 
or $\underline{c}_{ii}^t\le0$ for some $i$
\textbf{or} $\underline{q}_i^t\le0$ for some $i$,}
\Return $\varnothing$
\EndIf
\EndFor
\For{$i\in[K,...,1]$}
\State Compute $\underline{s}_i^t,\overline{s}_i^t$ using (\ref{eqn:slowerupper}), and $\underline{w}_i,\overline{w}_i$ using (\ref{eqn:wlowerupper})
\EndFor
\EndFor
\State \Return $W=\prod_{i=1}^k[\underline{w}_i,\overline{w}_i]$
\EndProcedure
\Procedure{IWPredictionSet}{$S_m,f,W=\prod_{k=1}^K[\underline{w}_k,\overline{w}_k],\epsilon,\delta$}
\State $V\sim\text{Uniform}([0,1])^m$
\State \Return $\hat\tau(S_m,V,W,\max_{k\in[K]}\overline{w}_k,\epsilon,\delta)$ as in (\ref{eqn:ipw})
\EndProcedure
\end{algorithmic}
\end{algorithm*}

Algorithm~\ref{alg:main} summarizes our approach. As can be seen, the coverage levels for the individual Clopper-Pearson intervals are chosen to satisfy the overall $1-\delta$ coverage guarantee. 
In particular, the PAC guarantee \eqref{eqn:labelpac} follows from 
\eqref{thm:ipw}, 
\eqref{lem:cp}, Lemma~\ref{lem:gaussianelim}, and a union bound.
\begin{theorem}[PAC Prediction Sets under Label Shift]
For any given $\epsilon,\delta\in(0, 1)$, 
under Assumptions \ref{assump:a1}, \ref{assump:a2} and \ref{assump:a3}, 
if $\forall i,j,t\in[K]$, we have
$\underline{c}_{ij}^t\ge0$, $\underline{c}_{ii}^t>0$, and $\underline{q}_i^t\ge0$, then
Algorithm~\ref{alg:main} satisfies
\begin{align*}
\mathbb{P}_{S_m\sim P^m,T_n^X\sim Q^n,V\sim U^m}\left[\mathbb{P}_{(X,Y)\sim Q}[Y\in C_{\hat\tau(S_m,V,W,b)}(X)]\ge1-\epsilon\right]\ge1-\delta.
\end{align*}
\end{theorem}
As discussed in Appendix~\ref{sec:positivity}, we can remove the requirement that $\underline{c}_{ij}^t\ge0$ and  $\underline{q}_i^t\ge0$ for $i\neq j$.

\section{Experiments}

\subsection{Experimental Setup}


\textbf{Predictive models.} 
We analyze five datasets: the CDC Heart dataset, CIFAR-10, Chest X-ray, AG News, and the Entity-13 dataset; their details are provided in Section \ref{rd}.
We use a two-layer MLP for the CDC Heart data with an SGD optimizer having a learning rate of $0.03$ and a momentum of $0.9$, a batch size of $64$ for $30$ epochs. For CIFAR-10, we finetuned a pretrained ResNet50 \cite{he2016deep}, with a learning rate of $0.01$ for $56$ epochs. 
For the ChestX-ray14 dataset, we use a pre-trained CheXNet~\citep{rajpurkar2017chexnet} with a DenseNet121~\citep{huang2017densely} backbone with learning rate $0.0003$ for $2$ epochs. 
For AGNews, a pre-trained Electra sequence classifier fine-tuned for one epoch with an AdamW optimizer using a learning rate of $0.00001$ is used. For the Entity-13 dataset, we finetuned a pretrained ResNet50, with a learning rate of $0.01$ for $13$ epochs.

\textbf{Hyperparameter choices.} 
There are two user-specified parameters that control the guarantees, namely $\delta$ and $\epsilon$. 
In our experiments, we choose $\delta=5\times 10^{-4}$ to ensure that, over 100 independent datasets $S_m$, 
there is a 95\% probability that 
the error rate is not exceeded.
Specifically, 
this ensures that 
$\mathbb{P}_{(X,Y)\sim P} [Y \in C_{\hat\tau(S_m)}(X)] \geq 1-\epsilon$ holds for all 100 trials, 
with probability approximately $1-0.95^{1/100} \approx 5\times 10^{-4} $. 
We select multiple $\epsilon$s for each dataset in a grid search way, demonstrating the validity of our algorithm.

\textbf{Dataset construction.} We follow the  label shift simulation strategies from previous work~\citep{lipton2018detecting}. First, we split the original dataset into training, source base, and target base. We use the training dataset to train the score function. Given label distributions $P_Y$ and $Q_Y$, we generate the source dataset $S_m$, target dataset $T_n^X$, and a labeled size $o$ target test dataset (sampled from $Q$) by sampling with replacement from the corresponding base dataset. We consider two choices of $P_Y$ and $Q_Y$: (i) a \textbf{tweak-one} shift, where we assign a probability to one of the labels, and keep the remaining labels equally likely, and (ii) \textbf{arbitrary} shift, where we shift each probability as described later.
 
\textbf{Baselines.} We compare our approach (\textbf{PS-W}) with several baselines (see Appendix~\ref{sec:addalg}): 
\begin{itemize}
[topsep=0pt,itemsep=0ex,partopsep=1ex,parsep=0ex,leftmargin=3ex]
\item \textbf{PS:} PAC prediction sets that do not account for label shift~\citep{vovk2012conditional,park2019pac}. This does not come with PAC guarantees under label shift.
\item \textbf{WCP:} Weighted conformal prediction under label shift, which targets marginal coverage~\citep{podkopaev2021distribution}. This does not come with PAC guarantees under label shift either.
\item \textbf{PS-R:} PAC prediction sets that account for label shift but ignore uncertainty in the importance weights; which does not come with PAC guarantees under label shift come with.
\item \textbf{PS-C:} Addresses label shift via a conservative upper bound on the empirical loss (see Appendix~\ref{sec:addalg} for details). This is the only baseline to come with PAC guarantees under label shift.
\end{itemize}
We compare to other baselines in Appendix~\ref{sec:morebase}, and to an oracle with the true weights in Appendix~\ref{apx:true-iw}, more results for different hyperparameters are in Appendix~\ref{sec:more}.

\textbf{Metrics.} We measure performance via the prediction set error, i.e., the fraction of $(x, y) \sim Q$ such that $y \notin C_{\tau}(x)$; 
and the average prediction set size, i.e., the mean of $|C_{\tau}(x)|$, evaluated on the held-out test set. We report the results over $100$ independent repetitions, randomizing both dataset generation and our algorithm. 

\subsection{Results \& Discussion}
\label{rd}


\textbf{CDC heart.}
We use the CDC Heart dataset, a binary classification problem \citep{cdcheart}, 
to predict the risk of  heart attack given features such as level of exercise or weight. 
We consider both large and small shifts.  
\begin{figure}[h]
\centering
\begin{subfigure}[b]{0.48\textwidth}
\includegraphics[height=0.4\linewidth]{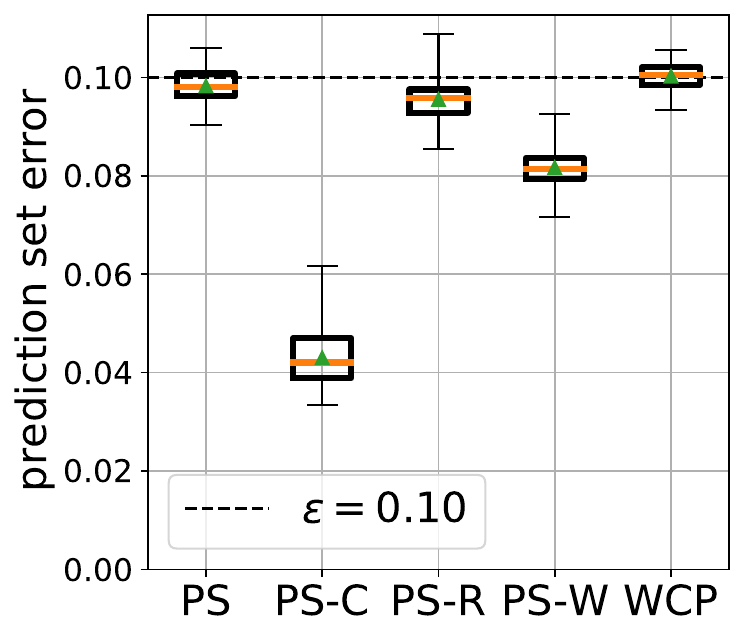}
\includegraphics[height=0.4\linewidth]{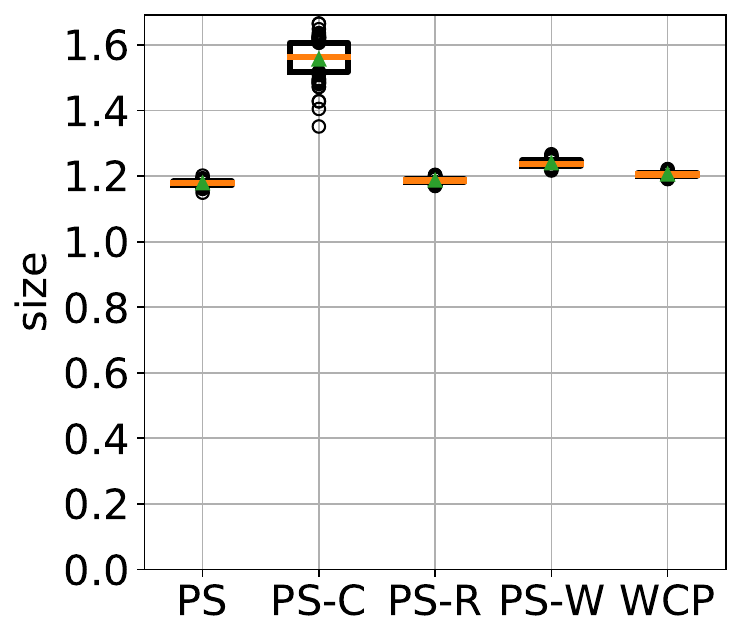}
\caption{Prediction set error and size under \textit{small} shifts on the CDC Heart dataset . Parameters are $\epsilon = 0.1$, $\delta = 5\times 10^{-4}$ , $m=42000$, $n=42000$, and $o= 9750$.}\label{fig:cdc1}
\end{subfigure}\quad
\begin{subfigure}[b]{0.49\textwidth}
\centering
\includegraphics[height=0.4\linewidth]{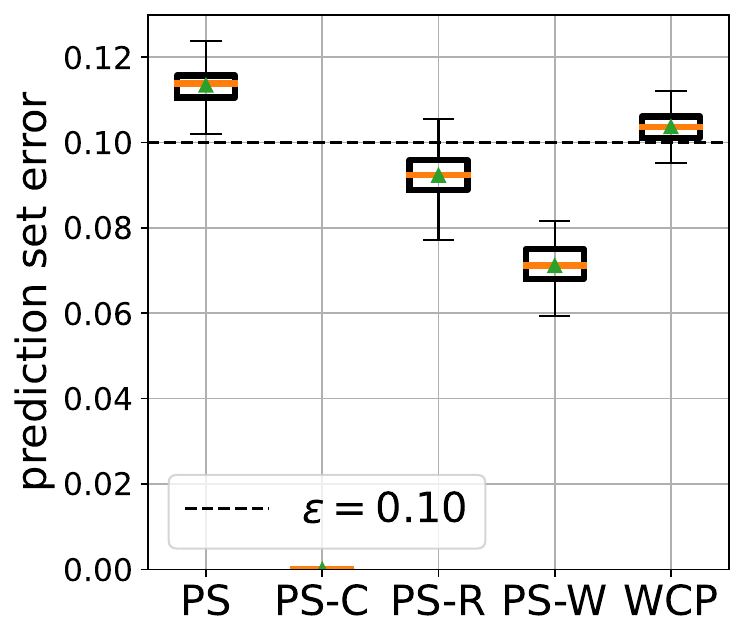}
\includegraphics[height=0.4\linewidth]{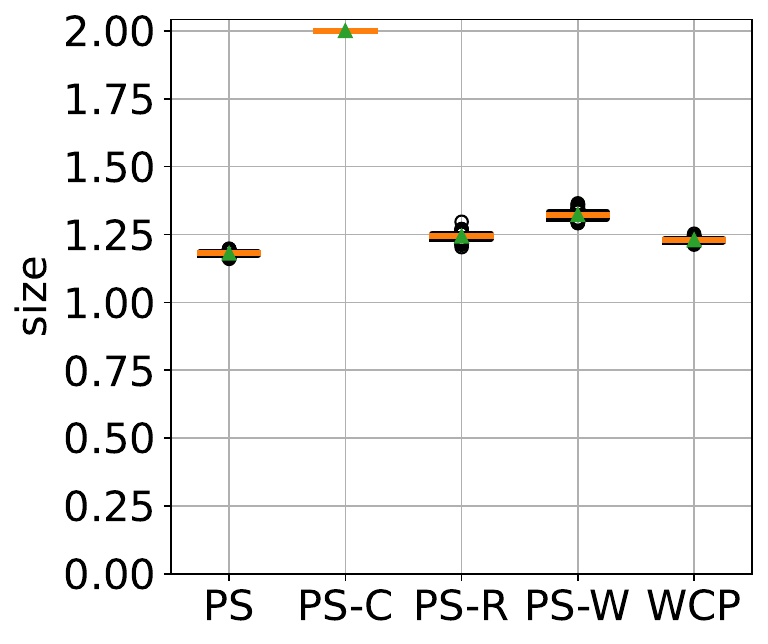}
\caption{Prediction set error and size under \textit{large} shifts on the CDC Heart dataset. Parameters are $\epsilon = 0.1$, $\delta = 5\times 10^{-4}$, $m=42000$, $n=42000$, and $o= 9750$.}\label{fig:cdc2}
\end{subfigure}
\vspace{-2ex}
\caption{Prediction set results on the CDC Heart dataset}
\end{figure}
For the large shift, the label distributions (denoted $(\text{pos}\%,\text{neg}\%)$) are $(94\%, 6\%)$ for source, and $(63.6\%, 36.4\%)$ for target; results are in Figure~\ref{fig:cdc2}.
We also consider a small shift with label distributions $(94\%, 6\%)$ for source, and $(91.3\%, 8.7\%)$ for target; results for $\epsilon = 0.1$ are in Figure~\ref{fig:cdc1}.
As can be seen, our PS-W algorithm satisfies the PAC guarantee while achieving smaller prediction set size than PS-C, the only baseline to satisfy the PAC guarantee. The PS and PS-R algorithms violate the PAC guarantee \footnote{The invisible error boxes are $0$ or extreme small values.}. 
%
%

\textbf{CIFAR-10.}
Next, we consider CIFAR-10~\citep{krizhevsky2009learning}, which has 10 labels. We consider a large and a small tweak-one shift. 
For the large shift, the label probability is 10\% for all labels in the source, 40.0\% for the tweaked label, and 6.7\% for other labels in the target; results are in Figure~\ref{fig:cifar1}. 
For small shifts, we use 10\% for all labels for the source, 18.2\% for the tweaked label, and 9.1\% for other labels for the target; results for $\epsilon = 0.1$ are in Figure~\ref{fig:cifar2}. 
Under large shifts, our PS-W algorithm satisfies the PAC guarantee while outperforming PS-C by a large margin. When the shift is very small, PS-W still satisfies the PAC guarantee, but generates more conservative prediction sets similar in size to those of PS-C (e.g., Figure~\ref{fig:cifar2}) given the limited data.

\begin{figure}[h]
\centering\small
\begin{subfigure}[b]{0.48\textwidth}
\includegraphics[height=0.4\linewidth]{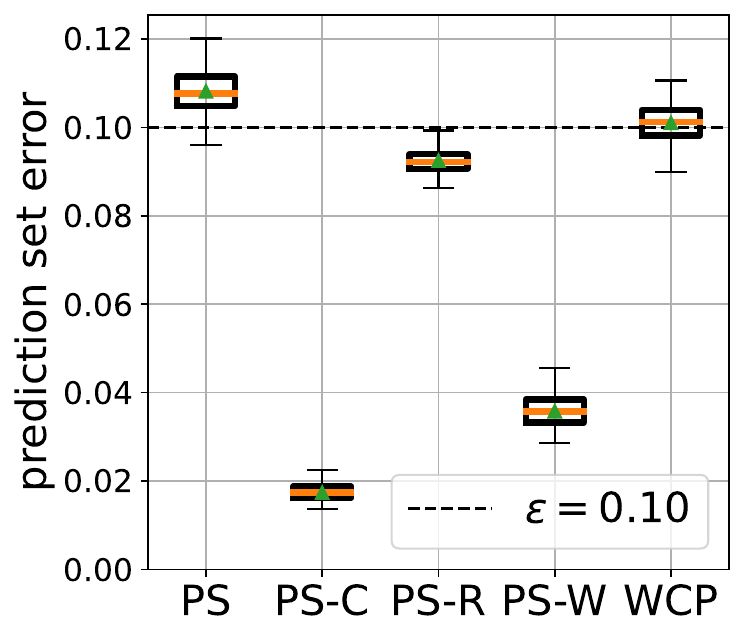}
\includegraphics[height=0.4\linewidth]{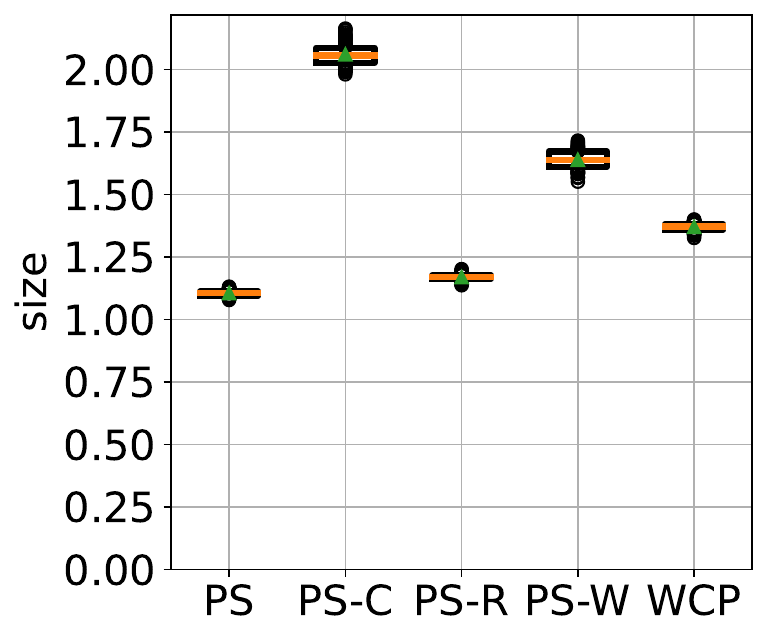}
\caption{Prediction set error and size with \textit{larger} shift on CIFAR-10. Parameters are $\epsilon = 0.1$, $\delta = 5\times 10^{-4}$, $m=27000$, $n=19997$, and $o=19997$.}\label{fig:cifar1}
\end{subfigure}
\quad
\begin{subfigure}[b]{0.48\textwidth}
\includegraphics[height=0.4\linewidth]{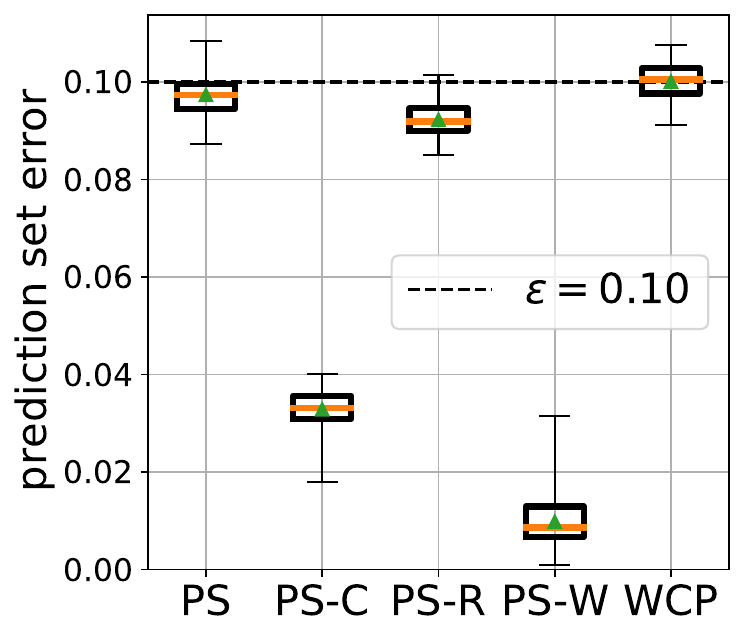}
\includegraphics[height=0.4\linewidth]{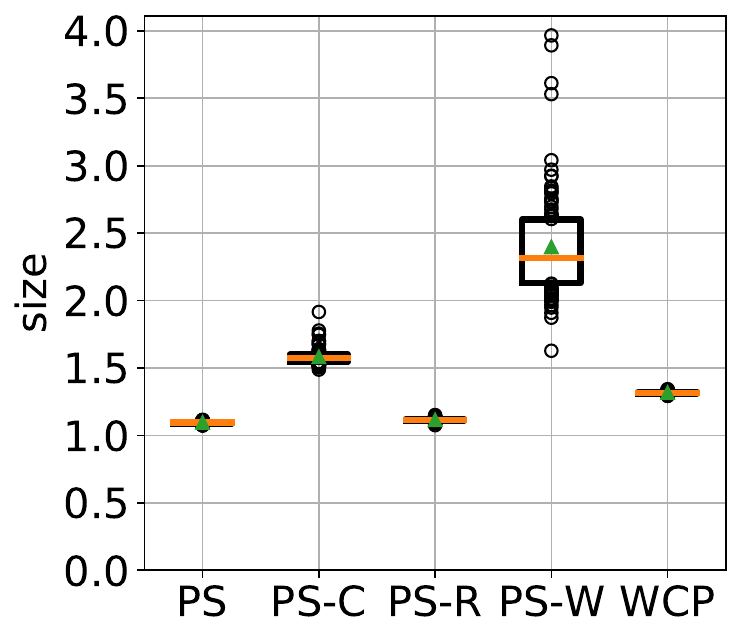}
\caption{Prediction set error and size with \textit{small} shift on CIFAR-10. Parameters are $\epsilon = 0.1$, $\delta = 5\times 10^{-4}$, $m=27000$, $n=16500$, and $o=16500$.}\label{fig:cifar2}
\end{subfigure}
\caption{Prediction set results on CIFAR-10.}
\end{figure}


\textbf{AGNews.}
AG News is a subset of AG's corpus of news articles~\citep{zhang2015character}. It is a text classification dataset with four labels: World, Sports, Business, and Sci/Tech. It contains 31,900 unique examples for each class. We use $\epsilon = 0.05$ and tweak-one label shifts.
We consider a large shift and a medium-sized calibration dataset, with label distributions equalling $(30.8\%, 30.8\%, 7.7\%, 30.8\%)$ for the source, and $(12.5\%, 12.5\%, 62.5\%, 12.5\%)$ for the target; results are in Figure \ref{fig:ag}. As before, our PS-W approach satisfies the PAC guarantee while achieving smaller set sizes than PS-C.

\begin{figure}[t]
\centering
\begin{subfigure}[b]{0.48\textwidth}
\includegraphics[height=0.4\linewidth]{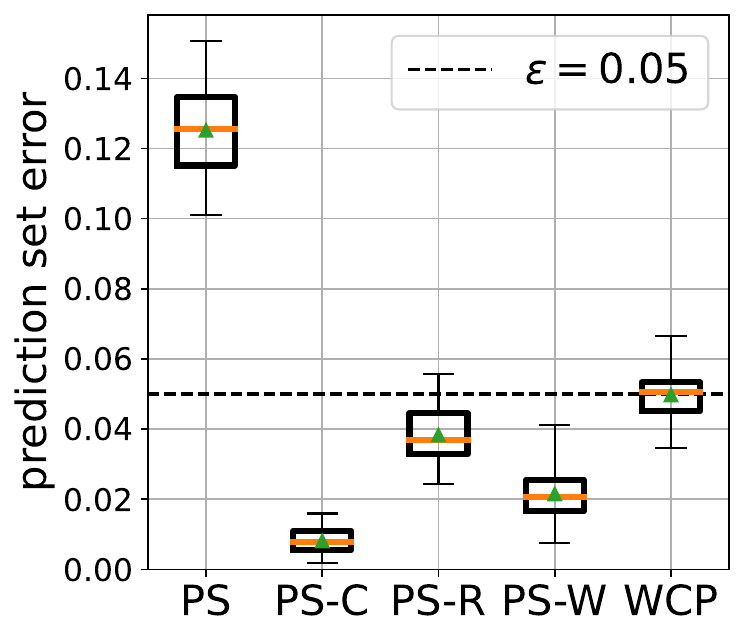}
\includegraphics[height=0.4\linewidth]{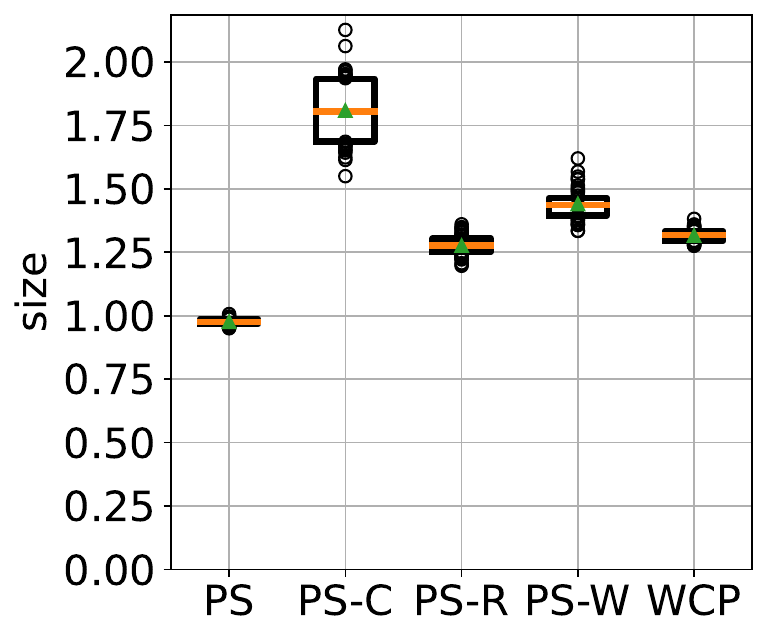}
\caption{Prediction set error and size on the AGNews Dataset. Parameters are $\epsilon = 0.05$, $\delta = 5\times 10^{-4}$, $m=26000$, $n=12800$, and $o=12800$.}\label{fig:ag}
\end{subfigure}
\quad
\begin{subfigure}[b]{0.485\textwidth}
\includegraphics[height=0.4\linewidth]{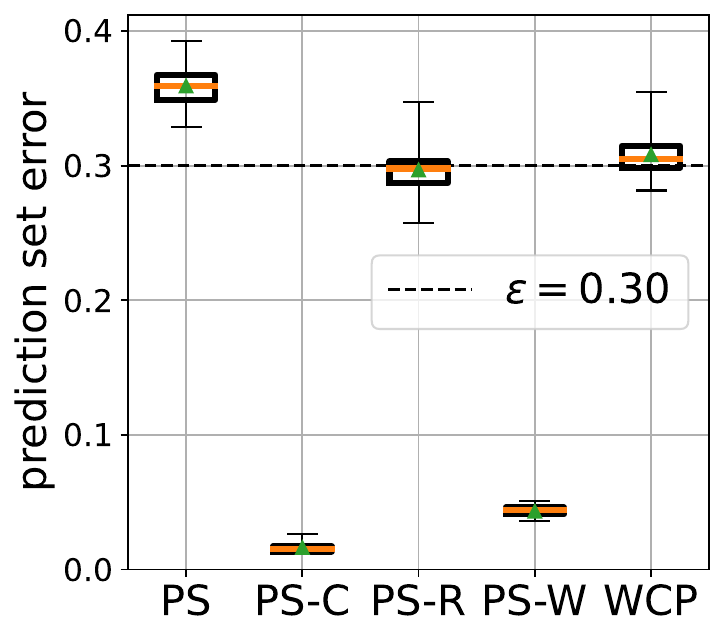}
\includegraphics[height=0.4\linewidth]{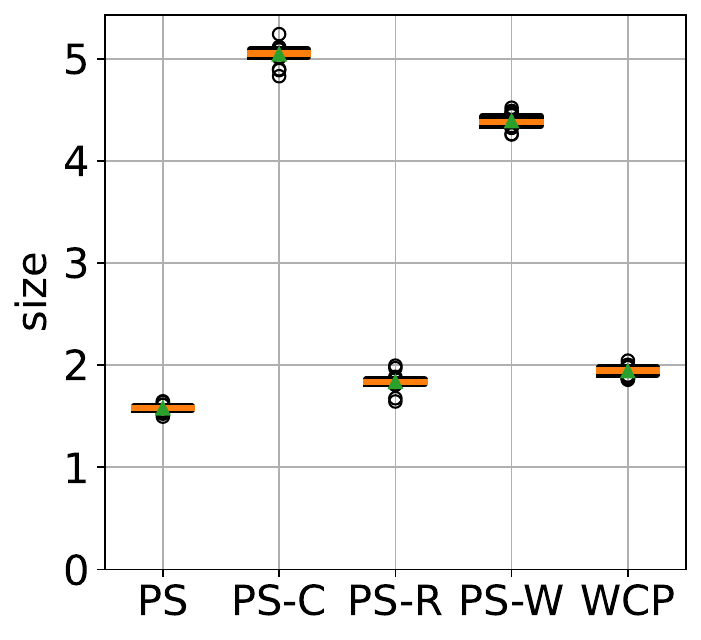}
\caption{Prediction set error and size on the ChestX-ray dataset. Parameters are $\epsilon = 0.3$, $\delta = 5\times 10^{-4}$, $m=67200$, $n=35200$, and $o=3520$.}\label{fig:ch}
\end{subfigure}
\end{figure}

\textbf{Entity-13} Entity-13 is a part of the BREEDs benchmark \citep{santurkar2020breeds}, which leverages the class hierarchy in ImageNet \citep{deng2009imagenet} to repurpose the original classes into superclasses Entity-13 contains 13 superclasses and a total of 390k images. It is also included in a recent label shift benchmark for relaxed label shift, RLSBench \citep{garg2023rlsbench}. We consider a general label shift and, additionally, a small shift with a medium-sized calibration dataset. The label probabilities equal $7.7\%$ each in the source and $(7.1\%, [3.6\%]*4, 10.71\%, 3.6\%, 7.1\%, 43.69\%, [3.6\%]*4)$ in the target. Results for $\epsilon = 0.1$ is shown in Figure \ref{fig:entity13-main}. As before, our PS-W approach satisfies the PAC guarantee while outperforming PS-C. The PS-R and WCP methods violate the constraint.

\begin{figure}[h]
\centering
\includegraphics[height=0.21\linewidth]{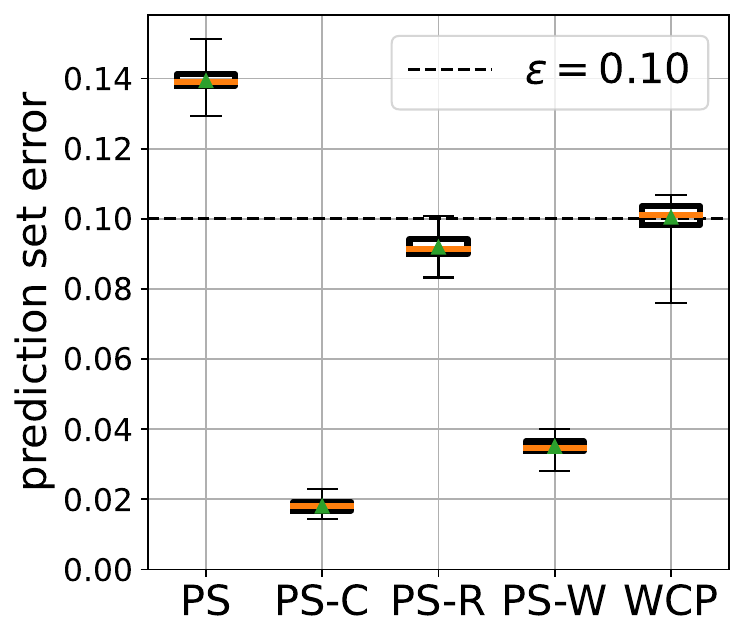}
\includegraphics[height=0.21\linewidth]{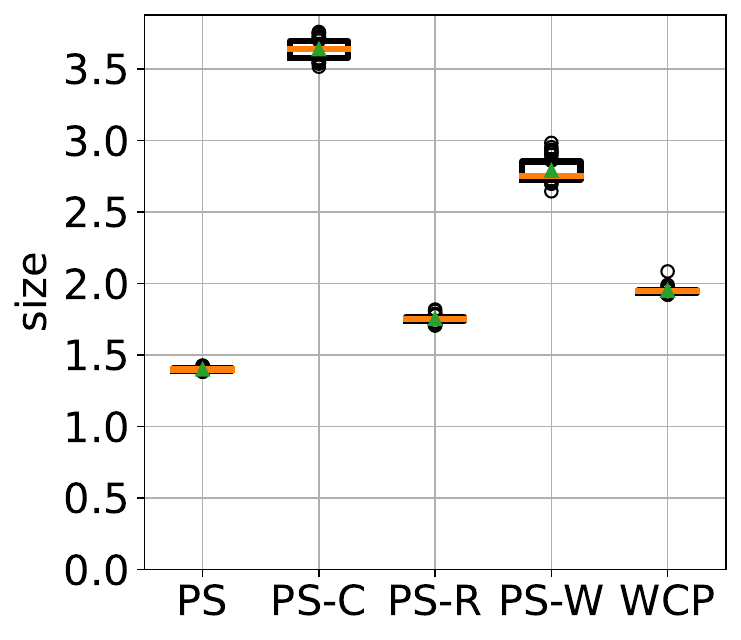}
\caption{Prediction set error and size on the Entity-13 dataset. Parameters are $\epsilon = 0.1$, $\delta = 5\times 10^{-4}$, $m=52000$, $n=21000$, and $o=4667$.}\label{fig:entity13-main}
\end{figure}

\textbf{ChestX-ray.}
ChestX-ray14~\citep{wang2017chestx} is a medical imaging dataset containing about 112K frontal-view X-ray images of 30K unique patients with fourteen disease labels. This dataset contains instances with multiple labels, which we omit. We also omit classes with few positively labeled datapoints, leaving $6$ classes: Atelectasis, Effusion, Infiltration, Mass, 
Nodule, Pneumothorax. 
We consider a large tweak-one shift, with label distributions of $([19.1\%]*4, 4.5\%, 19.1\%)$ for the source, and $([11.1\%]*4, 44.5\%, 11.1\%)$ for the target.
Results for $\epsilon = 0.3$ are in Figure \ref{fig:ch}. 
As before, our PS-W  approach satisfies the PAC guarantee while outperforming PS-C. 

\textbf{Discussion.}
In all our experiments, our approach satisfies the PAC guarantee; furthermore, it produces smaller prediction set sizes than PS-C---the only baseline to consistently satisfy the PAC guarantee---except when the label shift is small and the calibration dataset is limited. In contrast, the PS baseline does not account for label shift, and the PS-R baseline does not account for uncertainty in the importance weights, so they do not satisfy the PAC guarantee. The WCP baseline is designed to target a weaker guarantee, so it naturally does not satisfy the PAC guarantee. Thus, these results demonstrate the efficacy of our approach.

\vspace{-1ex}
\section{Conclusion}
\vspace{-1ex}
We have proposed a PAC prediction set algorithm for the label shift setting that accounts for confidence intervals around importance weights by modifying Gaussian elimination to propagate intervals. Our algorithm provides provable PAC guarantees and produces smaller prediction sets than several baselines that satisfy this condition. Experiments on five datasets demonstrate its effectiveness. Directions for future work include improving performance when the calibration dataset is small or when the label shift is small. 

\typeout{}
\bibliography{iclr2024_conference}
\bibliographystyle{iclr2024_conference}

\clearpage
\appendix
\onecolumn

\section{Additional Related Work}
\label{sec:related}


\textbf{Conformal Prediction.}
Our work falls into the broad area of 
distribution-free uncertainty quantification ~\citep{guttman1970statistical}. Specifically, it builds on ideas from conformal prediction~\citep{vovk2005algorithmic, balasubramanian2014conformal,angelopoulos2021gentle},
which aims to construct prediction sets with finite sample guarantees. 
The prediction sets are realized by a setting a threshold on top of a traditional single-label predictor (i.e. \emph{conformity/non-conformity scoring function}) and predicting all labels with scores above the threshold. 
Our approach is based on inductive conformal prediction~\citep{papadopoulos2002inductive, vovk2012conditional, lei2015conformal}, where the dataset is split into a training set for training the scoring function and a calibration set for constructing the prediction sets.

\textbf{PAC Prediction Sets.} 
Unlike conformal prediction methods that achieve a coverage guarantee of the median coverage performance.
PAC prediction sets consider training-data conditional correctness~\citep{vovk2012conditional,park2019pac}. That is, we achieve a $(\epsilon, \delta)$-guarantee that our performance on all future possible examples would only exceed the desired error rate level $\epsilon$ with probability under $\delta$.
This guarantee is also equivalent to the ``content'' guarantee for tolerance regions~\citep{wilks1941determination, fraser1956nonparametric}. 
All in all, PAC prediction sets obtain a totally different statistical guarantee compared to conformal prediction methods. The preference of the two methods depends on the application scenario. Obviously, PAC prediction sets favor the case when even a slight excess of the desired error rate would cause a significant safety concern.

\begin{table}[h]
\renewcommand{\arraystretch}{1.2}
\centering\small
\bgroup
\setlength{\tabcolsep}{0.2em}
\begin{tabular}{|c|c|c|}
\hline
    & Covariate Shift & Label Shift \\ \hline
Importance weights &   $q(x)/ p(x)$              &     $q(y)/p(y)$        \\ \hline
Invariance                                                  &   $p(y\mid x) = q(y \mid x)$\,              &     $p(x\mid y) = q(x\mid y)$\,        \\ \hline
\end{tabular}\caption{Comparison of covariate shift and label shift.}
\label{tbl:comp-shifts}
\egroup
\end{table}

\textbf{Label shift.}
Label shift~\citep{zadrozny2004learning, huang2006correcting, sugiyama2007direct, gretton2009covariate} is a kind of distribution shift; in contrast to the more widely studied covariate shift, it assumes the conditional covariate distribution is fixed but the label distribution may change; see Table~\ref{tbl:comp-shifts}. In more detail, letting $p$ and $q$ denote the probability density function of the source and target domains, respectively, we assume that $q(y)$ may be shifted compared to $p(y)$, but $p(x \mid y) = q(x \mid y)$. 
Label shift can arise when sets of classes change (e.g., their frequency increases) in scenarios like medical diagnosis and object recognition~\citep{storkey2009training, saerens2002adjusting, lipton2018detecting}. 

Early solutions required estimation of $q(x), p(x\mid y)$, often relying on kernel methods, which may scale poorly with data size and dimension~\citep{chan2005word, storkey2009training, zhang2013domain}. More recently, two approaches achieved scalability by assuming an approximate relationship 
between the classifier outputs $\hat{y}$ and ground truth labels $y$~\citep{lipton2018detecting, azizzadenesheli2019regularized, saerens2002adjusting}: Black Box Shift Estimation (BBSE)~\citep{lipton2018detecting} and RLLS~\citep{azizzadenesheli2019regularized} provided consistency results and finite-sample guarantees assuming the confusion matrix is invertible. Subsequent work developed a unified framework that decomposes the error in computing importance weights into miscalibration error and estimation error, with BBSE as a special case~\citep{garg2020unified}; this approach was extended to open set label shift domain adaptation~\citep{garg2022domain}.

\textbf{Conformal methods and distribution shift.} Due to its distribution-free nature, conformal prediction has been successfully dealing with distribution shift problem. Previous works mainly focus on covariate shift~\citep{tibshirani2019conformal, lei2021conformal,qiu2022distribution}.
\citep{podkopaev2021distribution} considers label shift; however, they assume that the true importance weights are exactly known, which is rarely the case in practice. While importance weights can be estimated,
there is typically uncertainty in these estimates that must be accounted for, especially in high dimensions.

The rigorous guarantee of PAC prediction sets requires quantifying the uncertainty in each part of the system. Thus, we apply a novel linear algebra method to obtain the importance weights confidence intervals for the uncertainty of label shift estimation. This setting has been studied in the setting of covariate shift~\citep{park2021pac}, but under a much simpler distribution structure.
In contrast, the data structure of label shift problem is more complex in the unsupervised label shift setting. 

More broadly, prediction sets have been studied in the meta-learning setting~\citep{dunn2018distribution,park2022pac}, as well as the setting of robustness to all distribution shifts with bounded $f$-divergence~\citep{cauchois2020robust}.

\textbf{Class-conditional prediction sets.}
Although not designed specifically for solving the label shift problem, methods for class-conditional coverage~\citep{sadinle2019least} and adaptive prediction sets (APS)~\citep{romano2020classification} can improve robustness to label shifts. However, class-conditional coverage is a stronger guarantee that leads to prediction sets larger than for our algorithm, while APS does not satisfy a PAC guarantee; we compare to these approaches in Appendix~\ref{sec:morebase}.


\section{Background on PAC Prediction Sets}\label{sec:k}
Finding the maximum $\tau$ that satisfies (\ref{eqn:pac}) is equivalent to choosing $\tau$ such that the empirical error
\begin{equation*}
    \bar{L}_{S_m}(C_{\tau}):=\sum_{(x,y)\in S_m} \mathbbm{1}(y \notin C_{\tau}(x))
\end{equation*}
on the calibration set $S_m$ is bounded~\citep{vovk2005algorithmic,park2019pac}. Let $F(k; m, \epsilon)=\sum^k_{i=0} \binom mk \epsilon^i(1-\epsilon)^{m-i}$ be the cumulative distribution function (CDF) of the binomial distribution $\text{Binom}(m, \epsilon)$ with $m$ trials and success probability $\epsilon$ evaluated at $k$. Then, \citet{park2019pac} constructs $C_{\hat{\tau}}$ via
\begin{align}
&\hat{\tau} = \max_{\tau\in T} \tau \quad \text{subj. to} \quad \bar{L}_{S_m} (C_{\tau})\leq k(m, \epsilon, \delta) \label{eq:max}, 
\\
&\text{where}~~
k(m, \epsilon, \delta) = \max_{k\in \mathbb{N} \cup \{0\}} k ~~ \text{subj. to}~~ F(k; m, \epsilon) \leq \delta. \nonumber
\end{align}
Their approach is equivalent to the method from~\citet{vovk2012conditional}, but formulated in the language of learning theory. By viewing the prediction set as a binary classifier, the PAC guarantee via this construction can be connected to the Binomial distribution. Indeed, for fixed $C$, $\bar{L}_{S_m}(C)$ has distribution $\text{Bionm}(m, L_P(C))$, since $\mathbbm{1}(y \notin C(x))$ has a $\text{Bernoulli}(L_P (C))$ distribution when $(x, y)\sim P$. 
Thus, $k(m, \epsilon, \delta)$ defines a bound such that if
$L_P (C) \leq \epsilon$, 
then $\bar{L}_{S_m}(C)\leq  k(m, \epsilon, \delta)$ with probability at least $1-\delta$.

\section{Background on Prediction Sets Under Distribution Shift}
\label{psdsalg}

Here we demonstrate how to obtain prediction sets
given intervals $\underline{w}^*\le w^*\le\overline{w}^*$ around the true importance weights. This approach is based closely on the strategy in \citep{park2021pac} for constructing prediction sets under covariate shift, but adapts it to the label shift setting (indeed, our setting is simpler since there are finitely many importance weights). The key challenge is computing the importance weight intervals, which we describe in detail below.

Given the true importance weights $w^*$, one strategy would be to use rejection sampling~\citep{von195113, shapiro2003monte,rubinstein2016simulation} to subsample $S_m$ to obtain a dataset that effectively consists of $N\le m$ i.i.d. samples from $Q$ (here, $N$ is a random variable, but this turns out not to be an issue). Essentially, for each $(x_i,y_i)\in S_m$, we sample a random variable $V_i\sim\text{Uniform}([0,1])$, and then accept samples where $V_i\ge w_{y_i}^*/b$, where $b$ is an upper bound on $w_y^*$:
\begin{align*}
T_N(S_m,V,w^*,b)=\left\{(x_i,y_i)\in S_m\biggm\vert V_i\ge\frac{w_{y_i}^*}{b}\right\}.
\end{align*}
In our setting, we can take $b=\max_{y\in\mathcal{Y}}w_y^*$. Then, we return $\hat\tau(T_N(S_m,V,w^*,b))$. Since $T_N(S_m,V,w^*,b)$ consists of an i.i.d. sample from $Q$, we obtain the desired PAC guarantee (\ref{eqn:labelpac}).

In practice, we do not know the true importance weights $w^*$. Instead, suppose we can obtain intervals $W_y=[\underline{w}_y^*,\overline{w}_y^*]$ such that $w_y^*\in W_y$ with high probability. We let $W=\prod_{y\in\mathcal{Y}}W_y$, and assume $w^*\in W$ with probability at least $1-\delta$. The algorithm proposed in \citep{park2021pac} adjusts the above algorithm to conservatively account for this uncertainty---i.e., it chooses $\tau$ so the PAC guarantee (\ref{eqn:labelpac}) holds for \emph{any} importance weights $w\in W$:
\begin{align}
\label{eqn:ipw}
\hat\tau(S_m,V,W,b)=\min_{w\in W}\hat\tau(T_N(S_m,V,w,b)).
\end{align}
We minimize over $\tau$ since choosing smaller $\tau$ leads to larger prediction sets, which is more conservative. \citep{park2021pac} show how to compute (\ref{eqn:ipw}) efficiently. We have the following guarantee:
\begin{theorem}[Theorem 4 in \citep{park2021pac}]
\label{thm:ipw2}
Assume that $w^*\in W$. Letting $U=\text{Uniform}([0,1])$,
\begin{align*}
\mathbb{P}_{S_m\sim P^m,V\sim U^m}\left[\mathbb{P}_{(X,Y)\sim Q}[y\in C_{\hat\tau(S_m,V,W,b)}]\ge1-\epsilon\right]\ge1-\delta.
\end{align*}
\end{theorem}
In other words, the PAC guarantee (\ref{eqn:labelpac}) holds, with the modification that the outer probability includes the randomness over the samples $V\sim U^m$ used by our algorithm.

\section{Ensuring the Confidence Bounds at each Step}
\label{sec:positivity}

The diagonal elements $c_{kk}$ of
the confusion matrix of an accurate classifier, are typically much larger than the other elements.
Indeed, for an accurate classifier, 
the probabilities of correct predictions $c_{kk} = P(g(x)=k, y=k)$ are higher than those of incorrect predictions $c_{ik} := P(g(x)=i, y=k), i\neq k$. 
On the other hand, the Clopper-Pearson interval 
is expected to be short (for instance, the related Wald interval has length or order $1/\sqrt{m}$, where $m$ is the sample size).
Thus, we expect that 
\begin{equation}\label{eqn:update}
    \overline{c}^0_{ik} \ll \underline{c}^0_{kk}, k = 1, \dots, K, i \neq k.
\end{equation}
Without loss of generality,  we consider
\eqref{eqn:clower} as an example. 
In the Gaussian elimination process, recall that the update at step $t$ is
\begin{equation}\label{eqn:jk}
\underline{c}_{ij}^{t+1} =  \underline{c}_{ij}^t - \dfrac{\overline{c}^t_{ik}}{\underline{c}_{kk}^t}\overline{c}_{kj}^t \quad\text{if } i,j > k.
\end{equation} 
Combining with \eqref{eqn:update}, the factor by which the $k$-th row is multiplied is small, i.e., $\overline{c}^t_{ik}/\underline{c}^t_{kk} \ll 1$.
Thus the resulting $i$-th diagonal elements 
\begin{equation*}
\underline{c}_{ii}^{t+1} =  \underline{c}_{ii}^t - \dfrac{\overline{c}^t_{ik}}{\underline{c}_{kk}^t}\overline{c}_{ki}^t
\end{equation*}
change little after each elimination step, and are expected to remain positive. 
Next we discuss intervals for off-diagonal elements.

\textbf{Balanced classifier.} For a \emph{balanced classifier}, when  $c_{ik}$ and $c_{jk}$ are close for all $i,j$ such that $i\neq k$, $j\neq k$, 
since the factor $\overline{c}^t_{ik}/\underline{c}^t_{kk}$ is small, the  lower bound in \eqref{eqn:jk} is expected to be positive.

\textbf{Imbalanced classifier.} For the more general case of a possibly imbalanced classifier, $c_{ij}$ and $c_{kj}$ may not be close. 
This could cause non-positive bounds at certain steps, so the confidence interval may not be valid at the next steps; e.g., \eqref{ind} may fail.
However, note that since
$$c_{ik}^t \in [\underline{c}_{ik}^t,\overline{c}_{ik}^t],\quad
c_{kj}^t \in [\underline{c}_{kj}^t,\overline{c}_{kj}^t],$$
we have
$$c_{ik}^t c_{kj}^t \le
\max\{|\underline{c}_{ik}^t|,|\overline{c}_{ik}^t|\}
\cdot
\max\{|\underline{c}_{kj}^t|,|\overline{c}_{kj}^t|\}$$ 
and hence
$$\frac{c_{ik}^t c_{kj}^t}{c_{kk}^t} \le
\frac{\max\{|\underline{c}_{ik}^t|,|\overline{c}_{ik}^t|\}
\cdot
\max\{|\underline{c}_{kj}^t|,|\overline{c}_{kj}^t|\}}
{\underline{c}_{kk}^t}.$$
In fact, one can derive the even tighter bound 
$$\max\left(\dfrac{\overline{c}_{ik}^t\overline{c}_{kj}^t}{\underline{c}_{kk}^t}, \dfrac{\underline{c}_{ik}^t\overline{c}_{kj}^t}{\underline{c}_{kk}^t}, \dfrac{\underline{c}_{ik}^t\underline{c}_{kj}^t}{\underline{c}_{kk}^t},
\dfrac{\overline{c}_{ik}^t\underline{c}_{kj}^t}{\underline{c}_{kk}^t},\dfrac{\overline{c}_{ik}^t\overline{c}_{kj}^t}{\overline{c}_{kk}^t}, \dfrac{\underline{c}_{ik}^t\overline{c}_{kj}^t}{\overline{c}_{kk}^t}, \dfrac{\underline{c}_{ik}^t\underline{c}_{kj}^t}{\overline{c}_{kk}^t},
\dfrac{\overline{c}_{ik}^t\underline{c}_{kj}^t}{\overline{c}_{kk}^t}\right).$$ 
This can be checked by carefully going through all possible cases of positive and negative values of the bounds.
 Similar changes can be made to computing the upper bounds.

It is possible for our final interval $W$ to contain negative lower bounds due to loose element-wise intervals or other factors. 
Since importance weights are non-negative, 
negative importance weight bounds act the same way as zero lower bounds in rejection sampling, and preserve our guarantees.

Finally, 
the requirement of an accurate classifier
is already imposed by methods such as BBSE to ensure the invertibility of the confusion matrix. 
Therefore, our Gaussian elimination approach does not impose significantly stronger assumptions.

\section{Proof of Lemma~\ref{lem:gaussianelim}}
\label{sec:gaussianelimproof}

For the first phase, we prove by induction on $t$ that (\ref{eqn:invariant}) holds for all $t$. The base case $t=0$ holds by assumption. For the induction step, we focus on $\underline{c}_{ij}^{t+1}$; the remaining bounds $\overline{c}_{ij}^{t+1}$, $\underline{q}_k^{t+1}$, and $\overline{q}_k^{t+1}$ follow similarly. 
There are three sub-cases i, each corresponding to one of the update rules in (\ref{eqn:clower}). 
For the first update rule $\overline{c}_{ij}^{t+1}=0$, \eqref{eqn:invariant} follows since the Gaussian elimination algorithm guarantees that $c_{ij}^{t+1}=0$.
For the second and third update rules, \eqref{eqn:invariant}  follows by direct algebra and the induction hypothesis.
For instance, for $i,j>t$, \eqref{ind} holds,
and similarly
$c_{ij}^{t+1} \ge  \underline{c}_{ij}^{t+1}$.

For the second phase, the fact that $\underline{s}\le s\le\overline{s}$ and $\underline{w}^*\le w^*\le\overline{w}^*$ follows by a similar induction argument. Since Gaussian elimination guarantees that $w^*=\mathbf{C}_P^{-1}q^*$, and we have shown that $w^*\in W=\prod_{i=1}^K[\underline{w}_i,\overline{w}_i]$, the claim follows. $\qed$



\section{Oracle Importance Weight Results}\label{apx:true-iw}

Here, we show comparisons to an oracle that is given the ground truth importance weights (which are unknown and must be estimated in most practical applications). It uses rejection sampling according to these weights rather than conservatively accounting for uncertainty in the weights.

First, for the CDC heart dataset, we consider the following label distributions: source $(94\%, 6\%)$, target: $(63.6\%, 36.4\%)$; see Figure~\ref{oc} for the results. Second, for the CIFAR-10 dataset, we consider the following label distributions: source 
$(10\%, 10\%, 10\%, 10\%, 10\%$, $10\%, 10\%$, $10\%, 10\%, 10\%)$, target: $(6.7\%, 6.7\%, 6.7\%, 40.0\%,  6.7\%, 6.7\%, 6.7\%, 6.7\%, 6.7\%, 6.7\%)$; see Figure \ref{oc10} for results.
\begin{figure}[t]
\centering
\begin{subfigure}[b]{0.48\textwidth}
\includegraphics[height=0.42\linewidth]{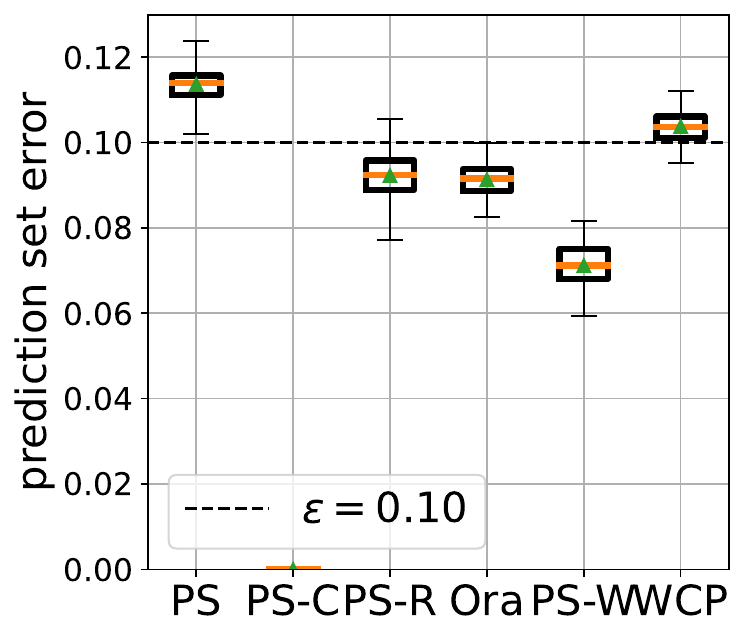}
\includegraphics[height=0.42\linewidth]{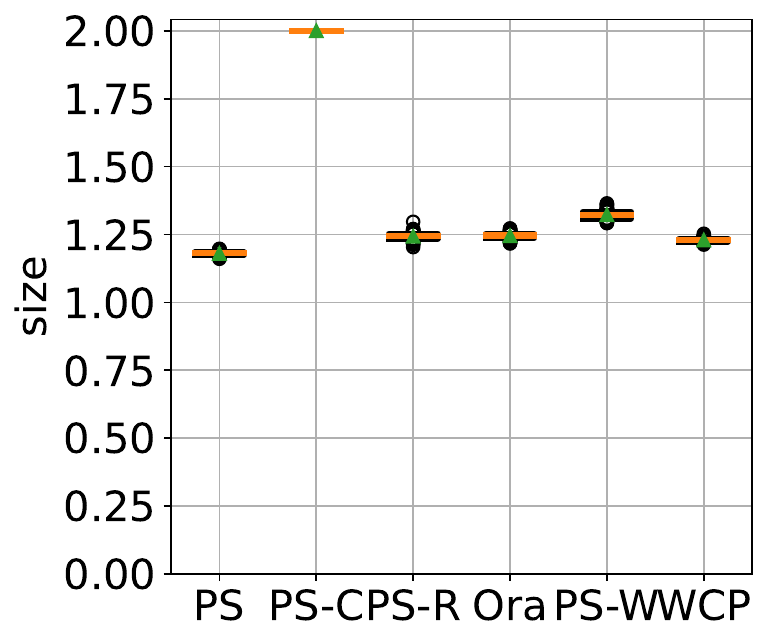}
\caption{Prediction set error and size on the CDC dataset. Parameters are $\epsilon = 0.1$, $\delta = 5\times 10^{-4}$, $m=42000$, $n=42000$, and $o= 9750$.}\label{oc}
\end{subfigure}
\quad
\begin{subfigure}[b]{0.485\textwidth}
\includegraphics[height=0.41\linewidth]{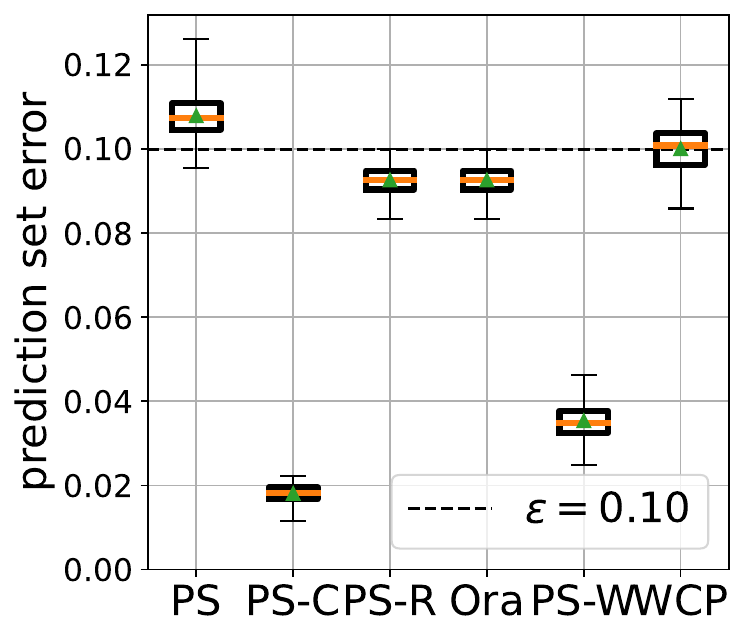}
\includegraphics[height=0.41\linewidth]{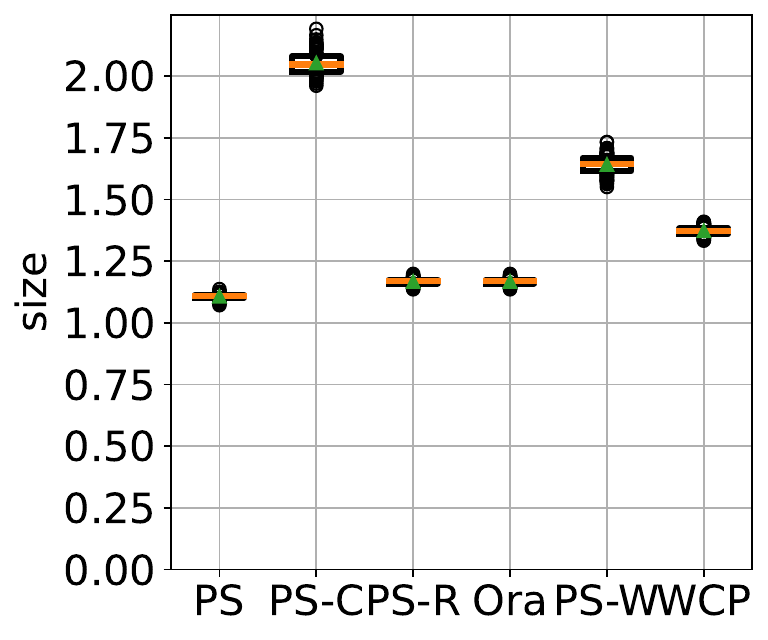}
\caption{Prediction set error and size on CIFAR10. Parameters are $\epsilon = 0.1$, $\delta = 5\times 10^{-4}$, $m=27000$, $n=19997$, and $o=19997$.}\label{oc10}
\end{subfigure}
\caption{Prediction set results with comparison to the oracle importance weight (Ora).}
\end{figure}

\section{Conservative Baseline}\label{sec:addalg}

We describe PS-C, the conservative baseline summarized in Algorithm~\ref{alg:psc}. In particular, given an upper bound $b\ge w^*$ on the importance weight, we use the upper bound
\begin{align*}
\mathbb{E}_{(X,Y)\sim P}[\ell(g(X),Y)\cdot w_Y^*]\le b\cdot \mathbb{E}_{(X,Y)\sim P}[\ell(g(X),Y)].
\end{align*}
As a consequence, we can run the original prediction set algorithm from \cite{vovk2012conditional,park2019pac} with a more conservative choice of $\epsilon$ that accounts for this upper bound.

\begin{algorithm*}
\caption{PS-C: an algorithm using the CP bound in \eqref{eq:maxiw}.}
\label{alg:psc}
\begin{algorithmic}[1]
\Procedure{PS-C}{$S_m, T^X_n, f, \mathcal{T}, \epsilon, \delta_w, \delta_C$} 
\State $\underline{c},\overline{c},\underline{q},\overline{q}\gets$ \Call{CPInterval}{$S_m, T^X_n, x\mapsto\operatorname*{\arg\max}_{y\in\mathcal{Y}}f(x,y), \delta_w$}
\State $W \gets$ \Call{IntervalGaussianElim}{$\underline{c},\overline{c},\underline{q},\overline{q}$}
\State $b \gets \max_{k\in[K]}\overline{w}_k$ 
\State \Return \Call{PS}{$S_m, f, \mathcal{T}, \epsilon / b, \delta_C$}
\EndProcedure
\end{algorithmic}
\end{algorithm*}

\begin{lemma}
Algorithm~\ref{alg:psc} satisfies the PAC guarantee under label shift \eqref{eqn:labelpac}.
\end{lemma}
\begin{proof}
Having constructed the importance weight intervals $w^*$, we can use $b = \max_{k\in[K]}\overline{w}_k$ to find a conservative upper bound on the risk as follows:
\begin{align}\small
&E_{(X,Y)\sim Q} [\mathbbm{1}(Y\notin C_{\tau}(X)] 
= \int q(x,y)\mathbbm{1}(y\notin C_{\tau}(x))dx dy\nonumber\\
&= \int p(x, y)w(y)\mathbbm{1}(y\notin C_{\tau}(x))dx dy 
\leq b E_{(X,Y)\sim P} [\mathbbm{1}(Y\notin C_{\tau}(X)].\label{eq:maxiw}
\end{align}
Hence, using the PS prediction set algorithm with parameters $(\ep/b,\delta)$, the output is $(\ep,\delta)$-PAC.
\end{proof}

\section{More Results}\label{sec:more}
\subsection{CDC Heart}
\begin{figure}[h]
\centering
\begin{subfigure}[b]{0.48\textwidth}
\includegraphics[height=0.41\linewidth]{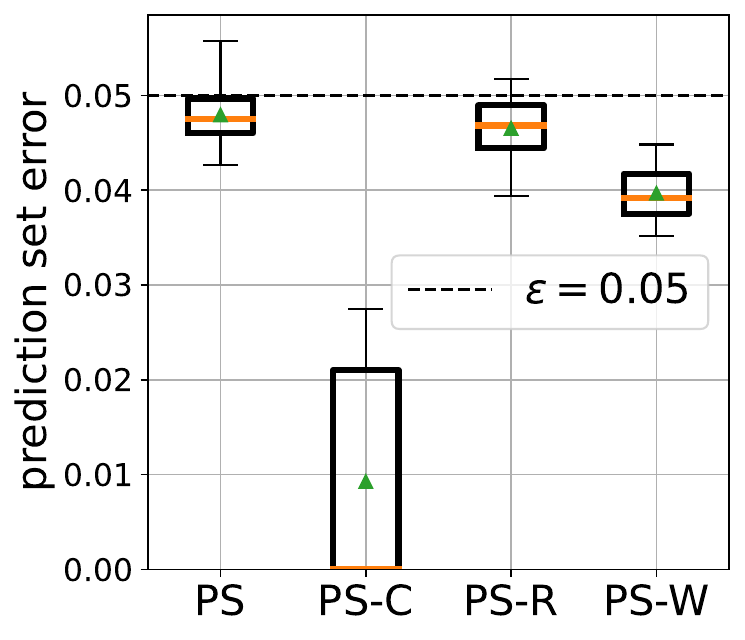}
\includegraphics[height=0.41\linewidth]{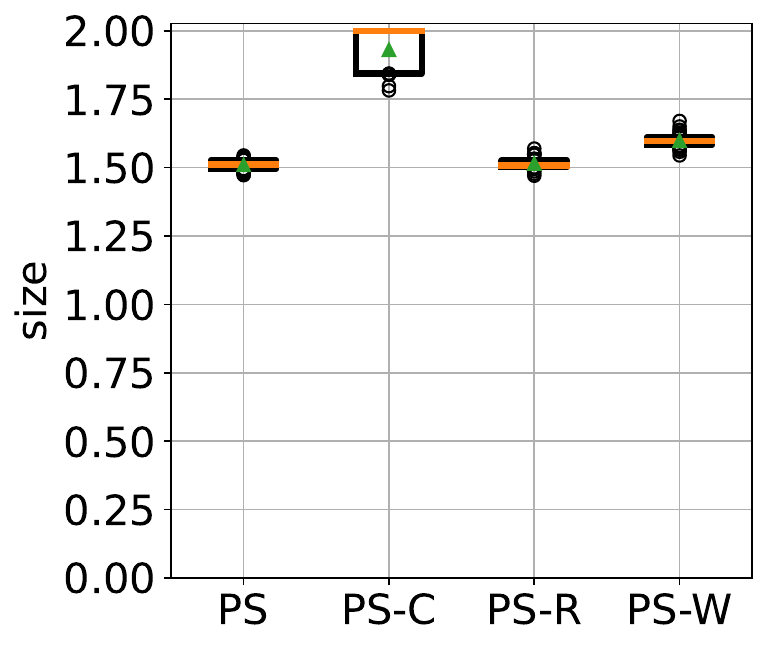}
\caption{Prediction set error and size under \textit{small} shifts on the CDC Heart dataset. Parameters are $\epsilon = 0.05$, $\delta = 5\times 10^{-4}$ , $m=42000$, $n=42000$, and $o= 9750$.}
\end{subfigure}\quad
\begin{subfigure}[b]{0.49\textwidth}
\centering
\includegraphics[height=0.41\linewidth]{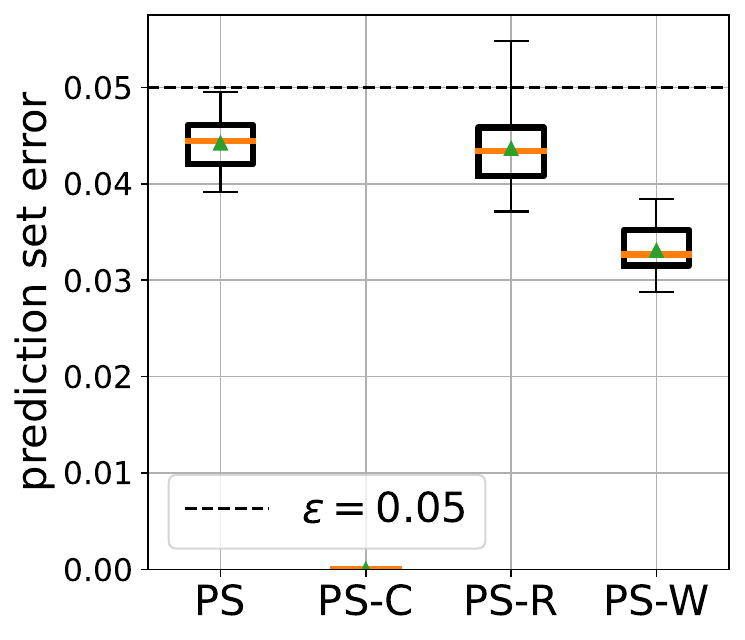}
\includegraphics[height=0.41\linewidth]{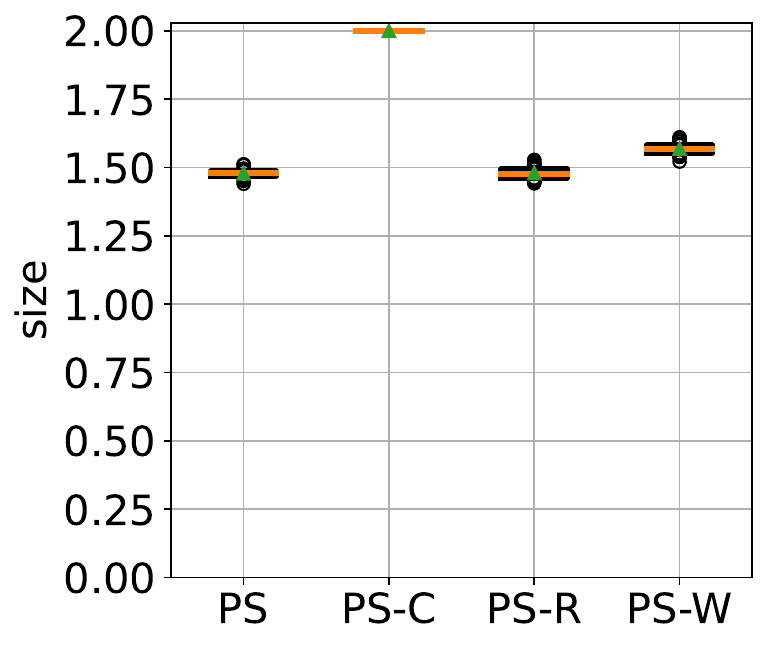}
\caption{Prediction set error and size under \textit{large} shifts on the CDC Heart dataset, Parameters are $\epsilon = 0.05$, $\delta = 5\times 10^{-4}$, $m=42000$, $n=42000$, and $o= 9750$.}
\end{subfigure}\vspace{2em}
\begin{subfigure}[b]{0.48\textwidth}
\includegraphics[height=0.41\linewidth]{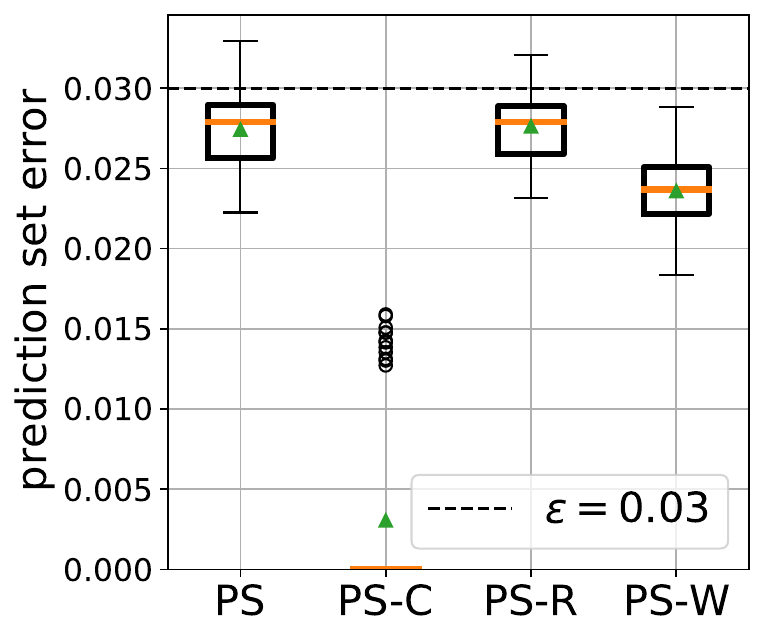}
\includegraphics[height=0.41\linewidth]{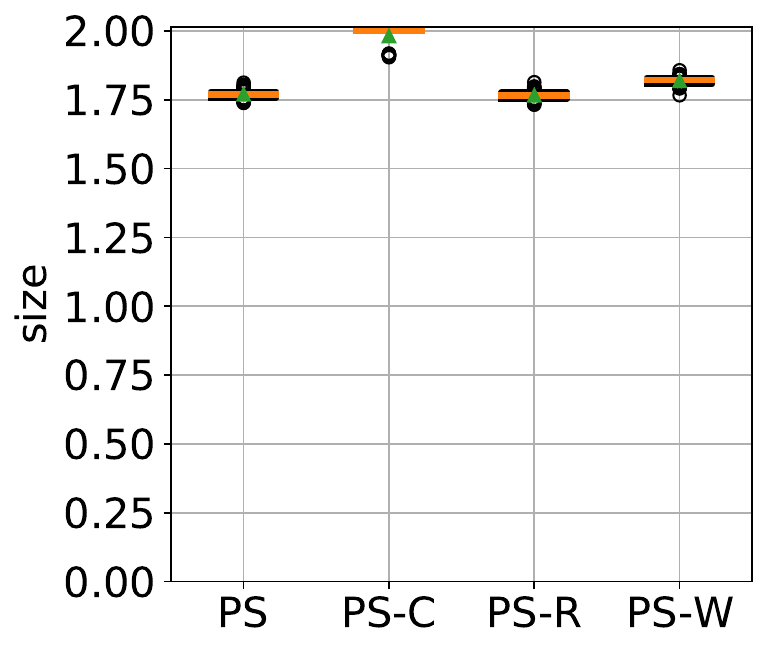}
\caption{Prediction set error and size under \textit{small} shifts on the CDC Heart dataset. Parameters are $\epsilon = 0.03$, $\delta = 5\times 10^{-4}$ , $m=42000$, $n=42000$, and $o= 9750$.}
\end{subfigure}\quad
\begin{subfigure}[b]{0.48\textwidth}
\centering
\includegraphics[height=0.41\linewidth]{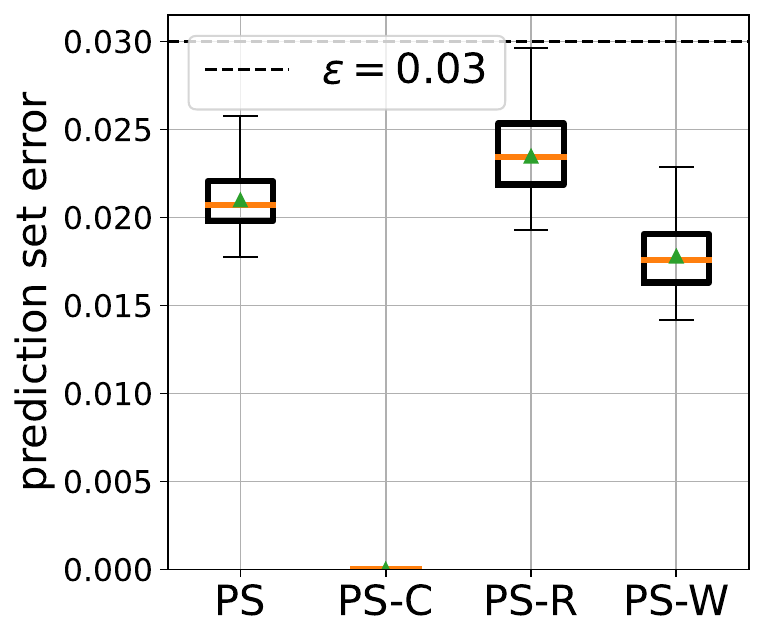}
\includegraphics[height=0.41\linewidth]{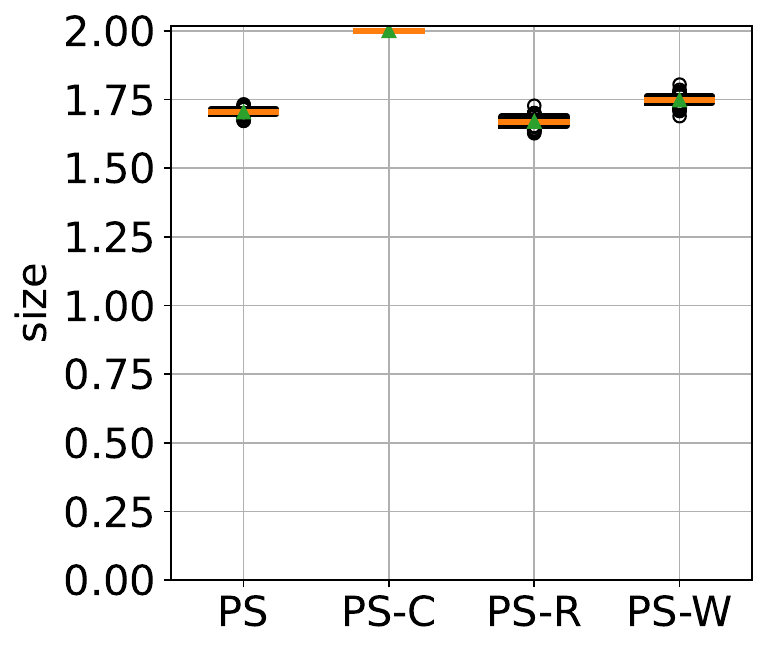}
\caption{Prediction set error and size under \textit{large} shifts on the CDC Heart dataset, Parameters are $\epsilon = 0.03$, $\delta = 5\times 10^{-4}$, $m=42000$, $n=42000$, and $o= 9750$.}
\end{subfigure}\vspace{2em}
\begin{subfigure}[b]{0.48\textwidth}
\includegraphics[height=0.41\linewidth]{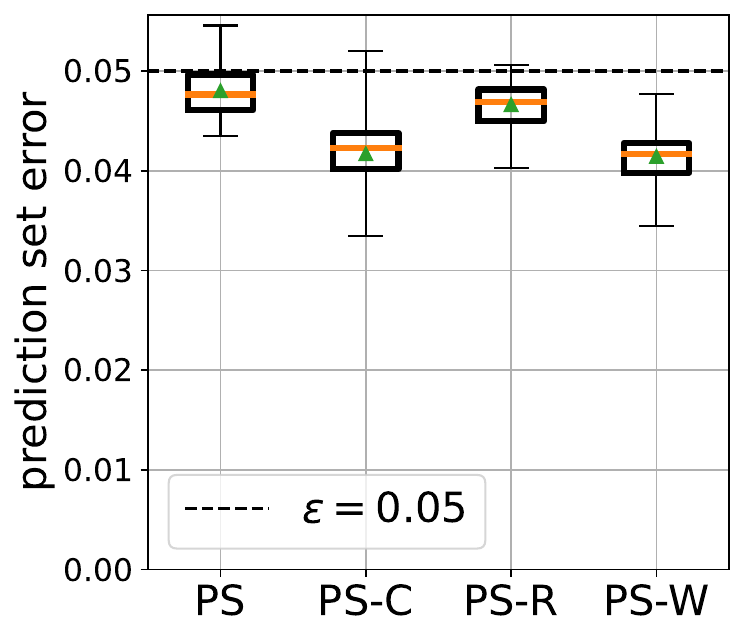}
\includegraphics[height=0.41\linewidth]{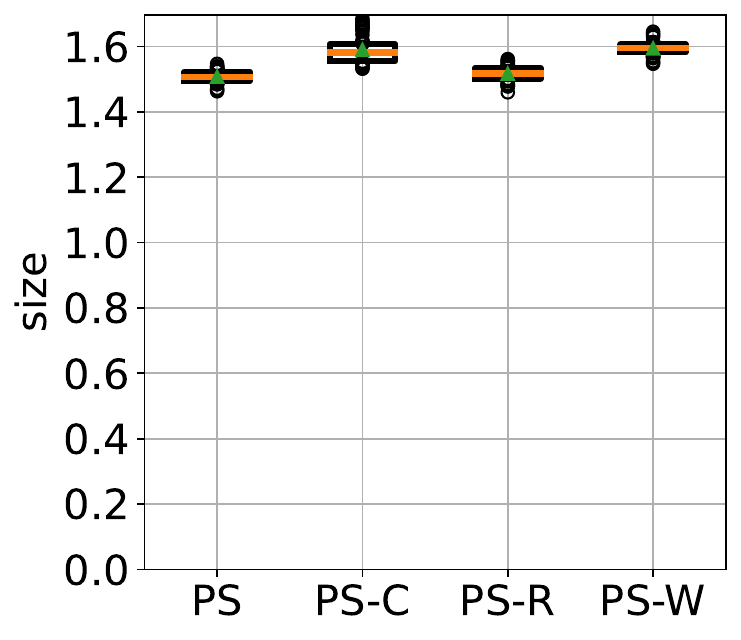}
\caption{Prediction set error and size under \textit{reversed small} shifts $((91.3\%, 8.7\%)\rightarrow(94\%, 6\%))$ on the CDC Heart dataset. Parameters are $\epsilon = 0.05$, $\delta = 5\times 10^{-4}$ , $m=42000$, $n=42000$, and $o= 9750$.}
\end{subfigure}
\begin{subfigure}[b]{0.48\textwidth}
\centering
\includegraphics[height=0.41\linewidth]{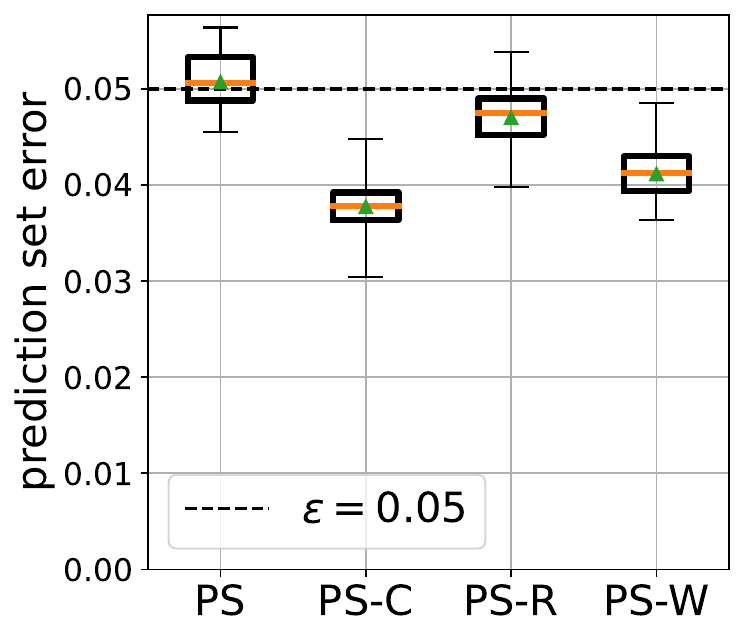}
\includegraphics[height=0.41\linewidth]{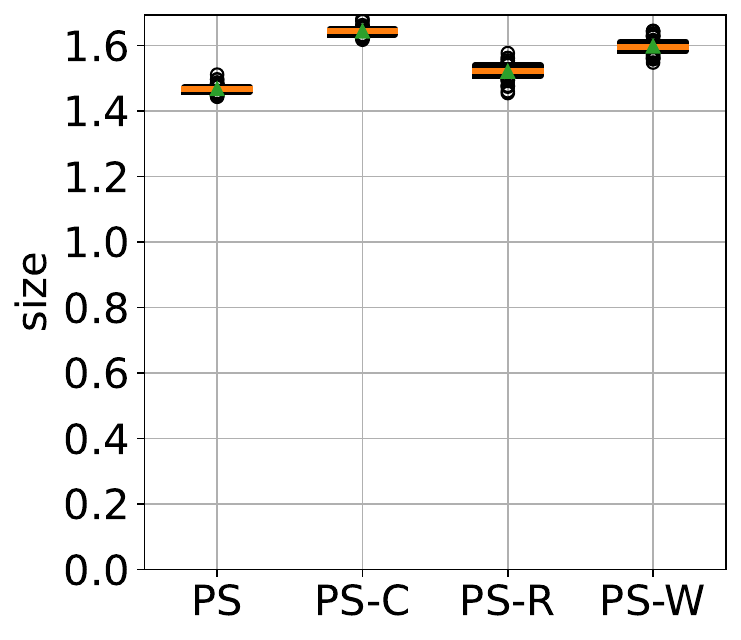}
\caption{Prediction set error and size under \textit{reversed large} shifts $((63.6\%, 36.4\%)\rightarrow(94\%, 6\%))$ on the CDC Heart dataset, Parameters are $\epsilon = 0.03$, $\delta = 5\times 10^{-4}$, $m=42000$, $n=42000$, and $o= 9750$.}
\end{subfigure}
\caption{More Prediction set results with different hyperparameters on the CDC Heart dataset.}
\end{figure}

\clearpage
\subsection{CIFAR-10}
\begin{figure}[h]
\centering
\begin{subfigure}[b]{0.48\textwidth}
\includegraphics[height=0.41\linewidth]{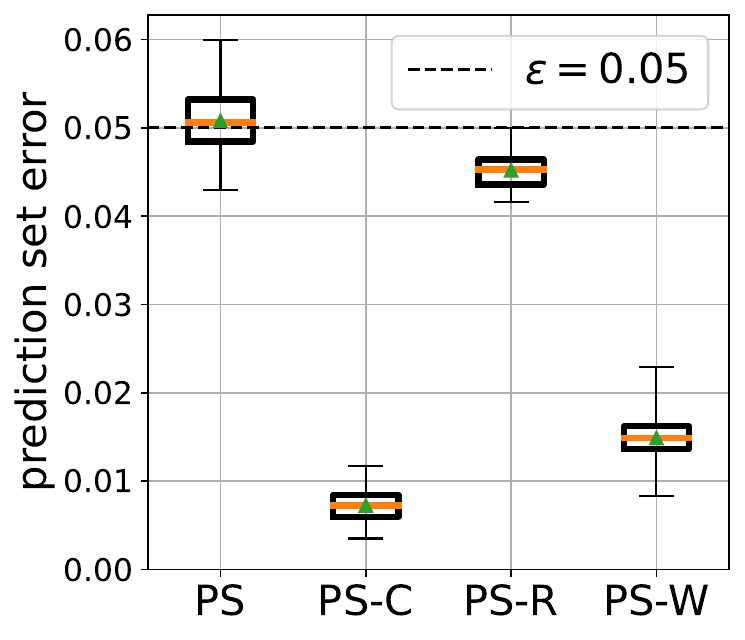}
\includegraphics[height=0.41\linewidth]{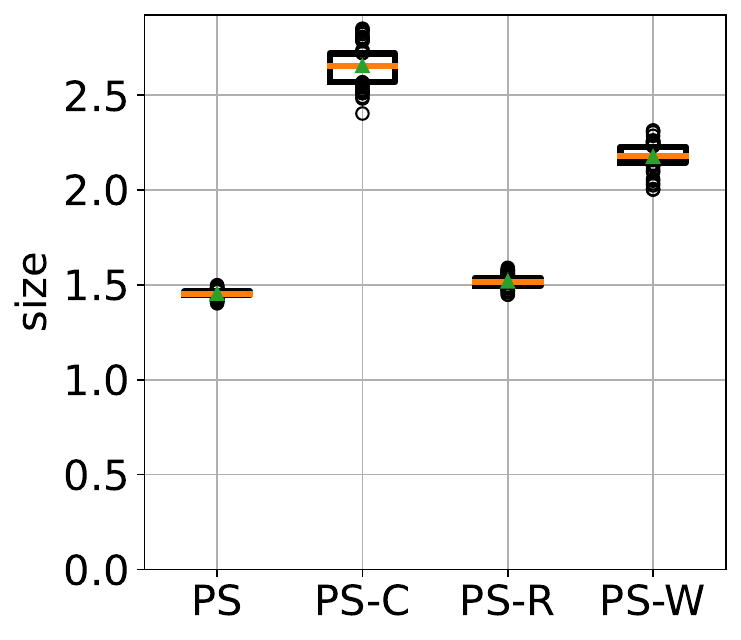}
\caption{Prediction set error and size with \textit{large} shifts on CIFAR-10. Parameters are $\epsilon = 0.05$, $\delta = 5\times 10^{-4}$, $m=27000$, $n=19997$, and $o=19997$.}
\end{subfigure}\quad
\begin{subfigure}[b]{0.49\textwidth}
\centering
\includegraphics[height=0.41\linewidth]{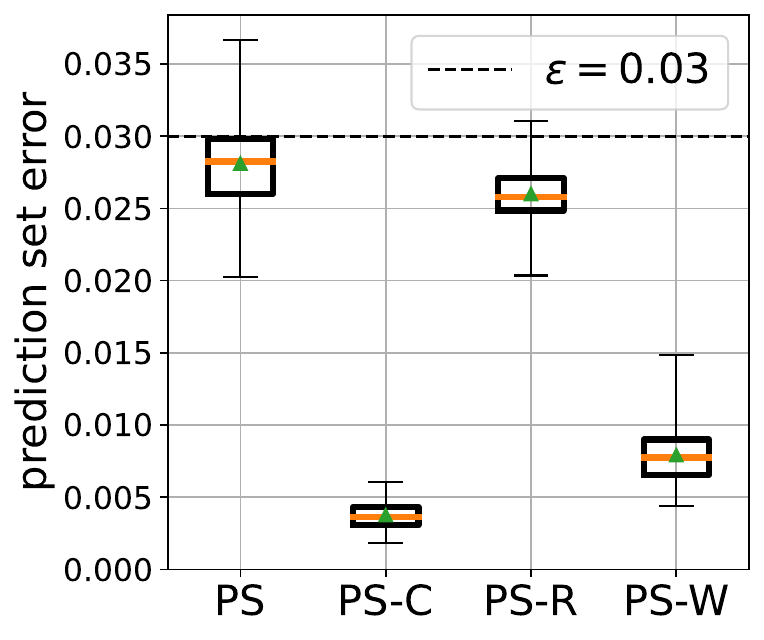}
\includegraphics[height=0.41\linewidth]{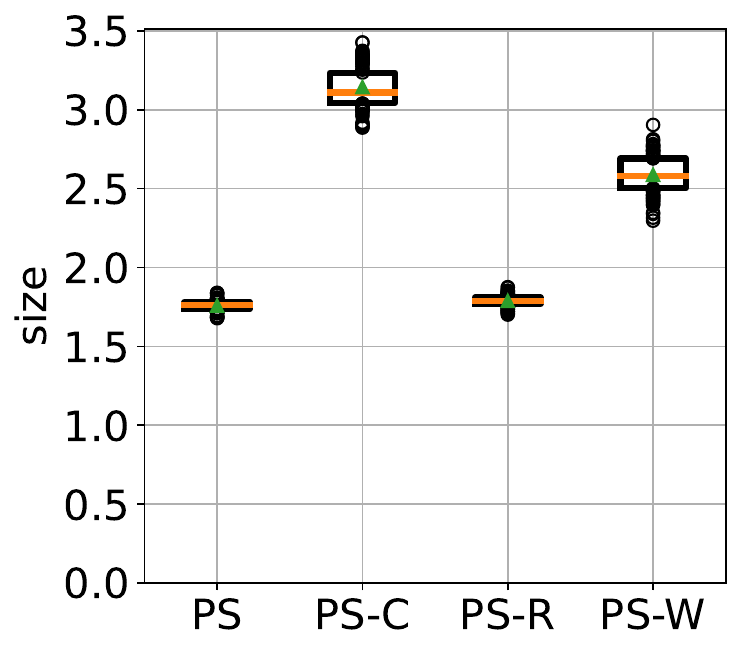}
\caption{Prediction set error and size with \textit{large} shifts on CIFAR-10. Parameters are $\epsilon = 0.03$, $\delta = 5\times 10^{-4}$, $m=27000$, $n=19997$, and $o=19997$.}
\end{subfigure}
\caption{More Prediction set results with different hyperparameters on the CIFAR-10 dataset.}
\end{figure}

\subsection{AGNews}
\begin{figure}[h]
\centering
\begin{subfigure}[b]{0.48\textwidth}
\includegraphics[height=0.42\linewidth]{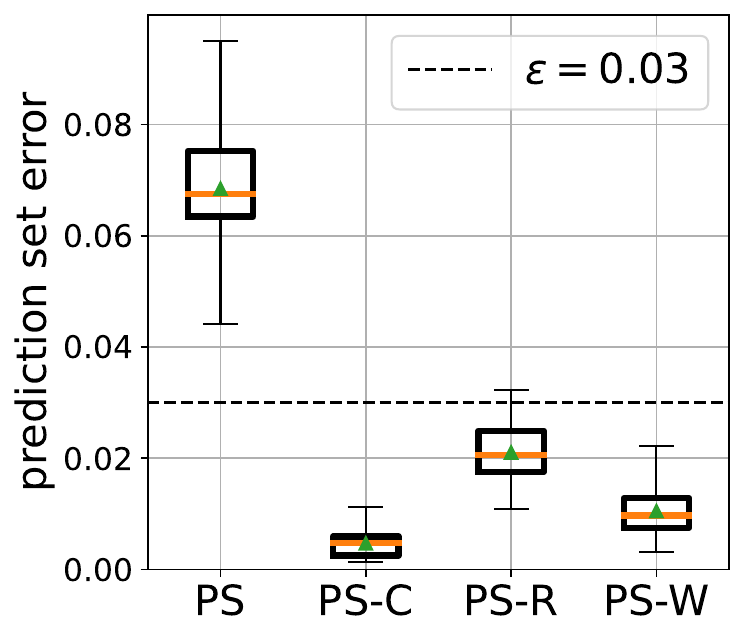}
\includegraphics[height=0.42\linewidth]{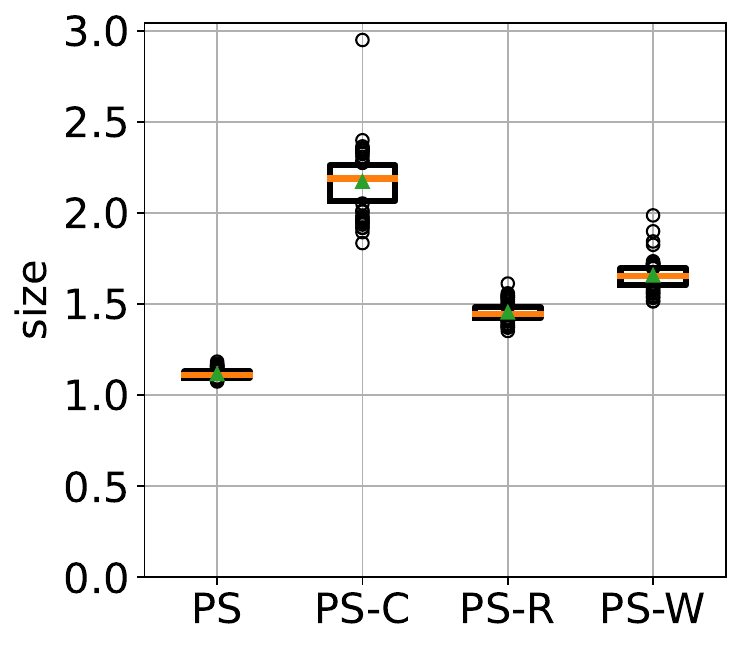}
\caption{Prediction set error and size on the AGNews Dataset. Parameters are $\epsilon = 0.03$, $\delta = 5\times 10^{-4}$, $m=26000$, $n=12800$, and $o=12800$.}
\end{subfigure}
\quad
\begin{subfigure}[b]{0.485\textwidth}
\includegraphics[height=0.41\linewidth]{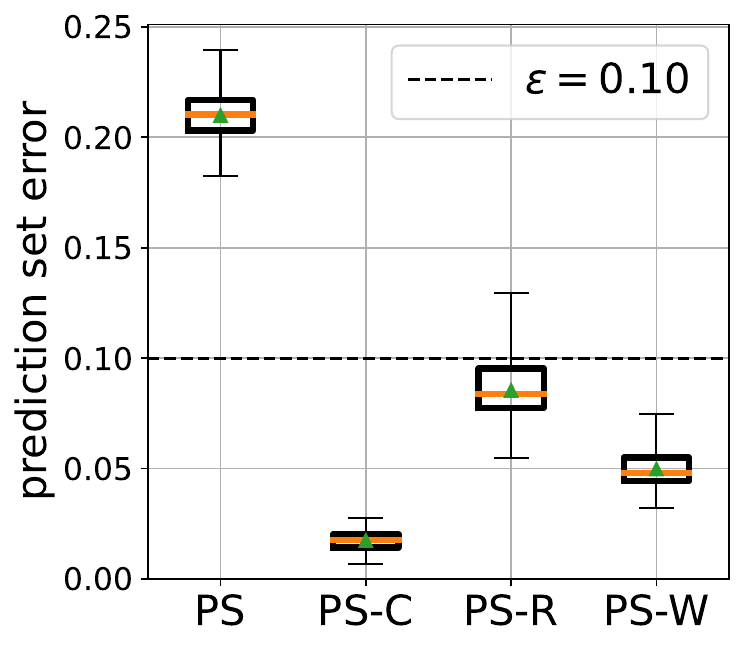}
\includegraphics[height=0.41\linewidth]{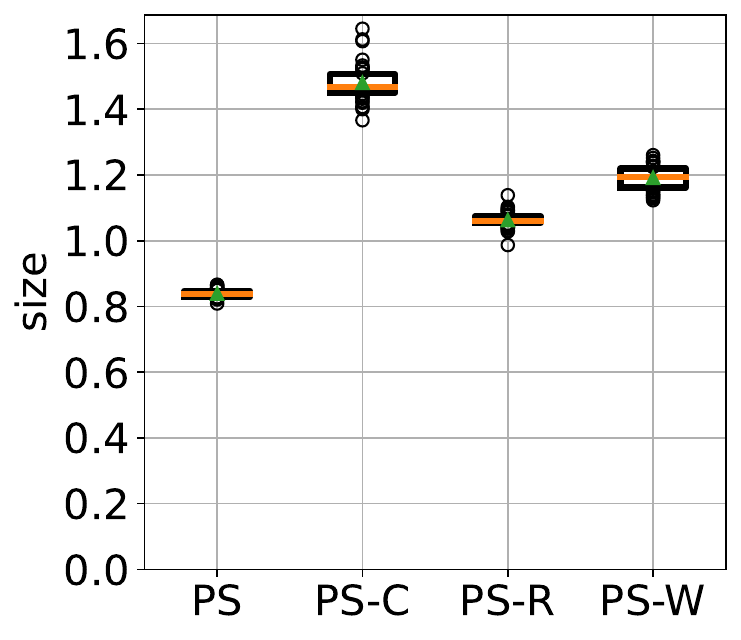}
\caption{Prediction set error and size on the AGNews Dataset. Parameters are $\epsilon = 0.1$, $\delta = 5\times 10^{-4}$, $m=26000$, $n=12800$, and $o=12800$.}
\end{subfigure}
\caption{More Prediction set results with different hyperparameters on the AGNews dataset.}
\end{figure}

\subsection{Entity-13}
\begin{figure}[h]
\centering
\begin{subfigure}[b]{0.48\textwidth}
\includegraphics[height=0.41\linewidth]{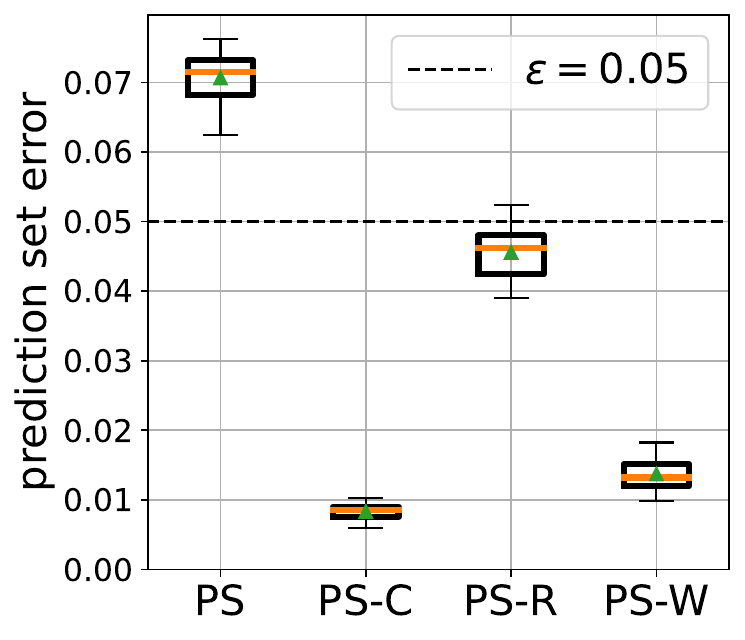}
\includegraphics[height=0.41\linewidth]{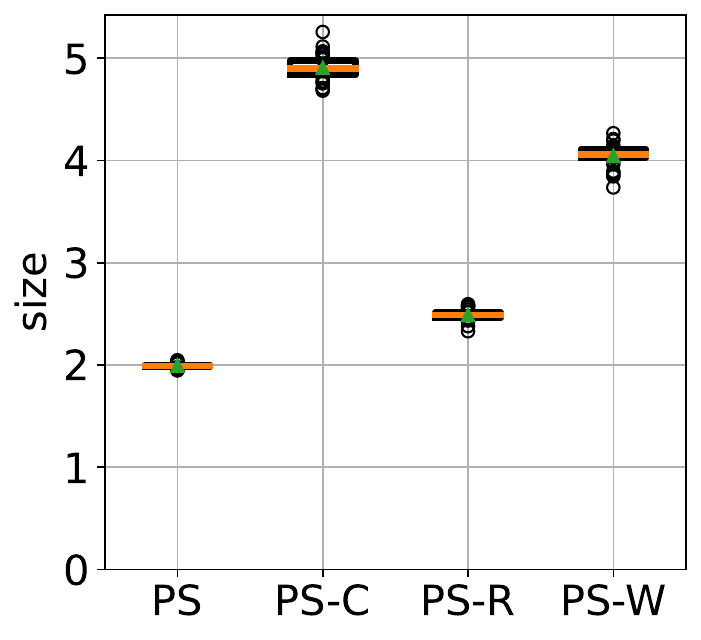}
\caption{Prediction set error and size on the Entity-13 dataset. Parameters are $\epsilon = 0.05$, $\delta = 5\times 10^{-4}$, $m=52000$, $n=21000$, and $o=4667$.}
\end{subfigure}
\quad
\begin{subfigure}[b]{0.48\textwidth}
\includegraphics[height=0.41\linewidth]{figures/entity13/entity0.1/plot_error_rnd_n_50000_eps_0.100000_delta_0.000500.pdf}
\includegraphics[height=0.41\linewidth]{figures/entity13/entity0.1/plot_size_rnd_n_50000_eps_0.100000_delta_0.000500.pdf}
\caption{Prediction set error and size on the Entity-13 dataset. Parameters are $\epsilon = 0.1$, $\delta = 5\times 10^{-4}$, $m=52000$, $n=21000$, and $o=4667$.}
\end{subfigure}
\begin{subfigure}[b]{0.48\textwidth}
\includegraphics[height=0.41\linewidth]{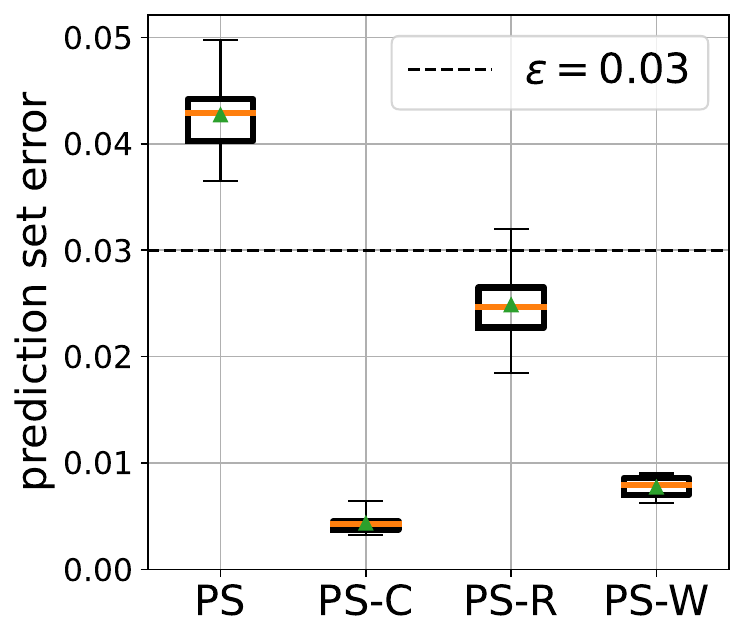}
\includegraphics[height=0.41\linewidth]{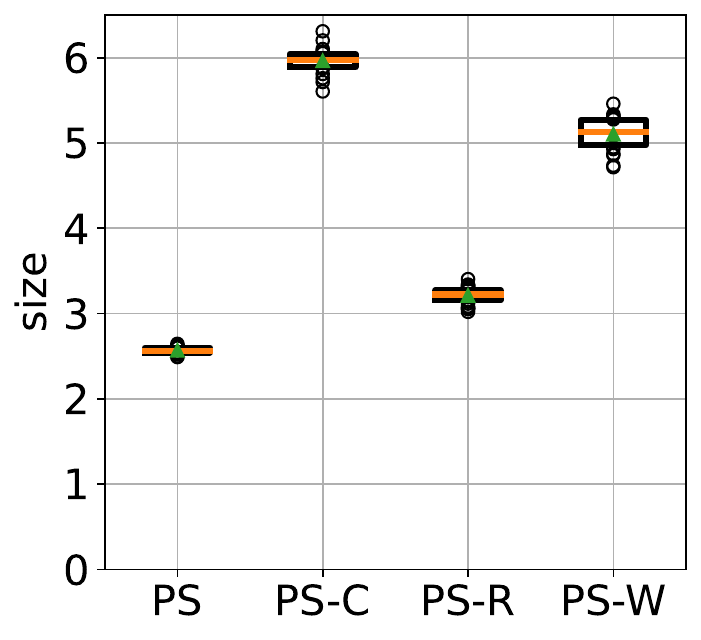}
\caption{Prediction set error and size on the Entity-13 dataset. Parameters are $\epsilon = 0.03$, $\delta = 5\times 10^{-4}$, $m=52000$, $n=21000$, and $o=4667$.}
\end{subfigure}
\caption{More Prediction set results with different hyperparameters on the Entity-13 dataset.}\label{fig:entity13}
\end{figure}

\subsection{ChestX-ray}
\begin{figure}[h]
\centering
\begin{subfigure}[b]{0.485\textwidth}
\includegraphics[height=0.4\linewidth]{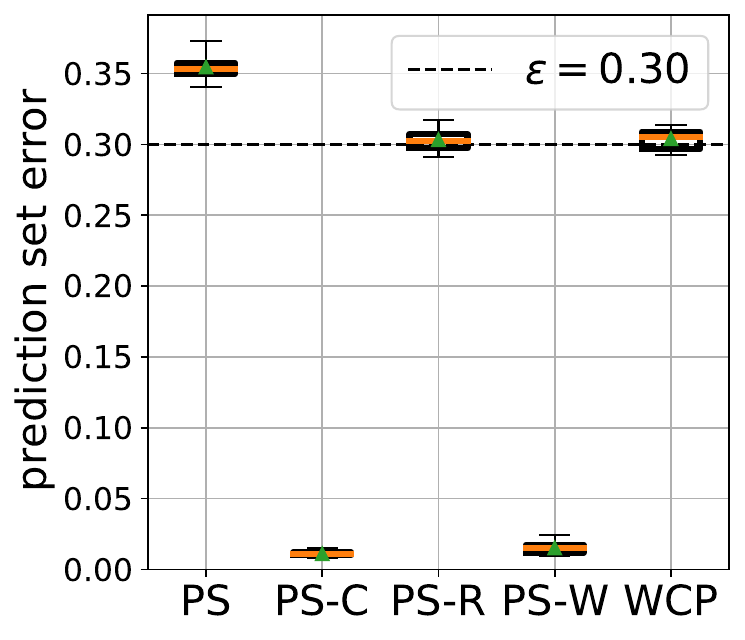}
\includegraphics[height=0.4\linewidth]{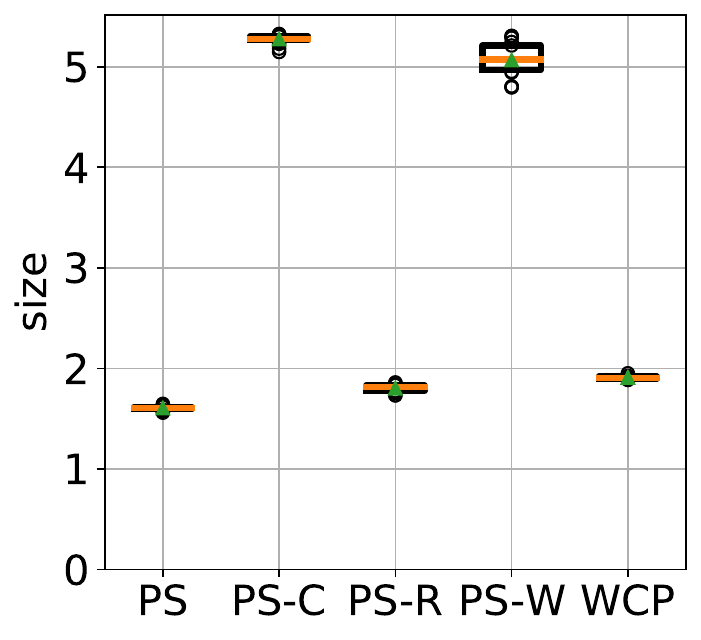}
\caption{Prediction set error and size on the ChestX-ray dataset. Parameters are $\epsilon = 0.3$, $\delta = 5\times 10^{-4}$, $m=33600$, $n=17600$, and $o=3520$.}
\end{subfigure}
\begin{subfigure}[b]{0.485\textwidth}
\includegraphics[height=0.4\linewidth]{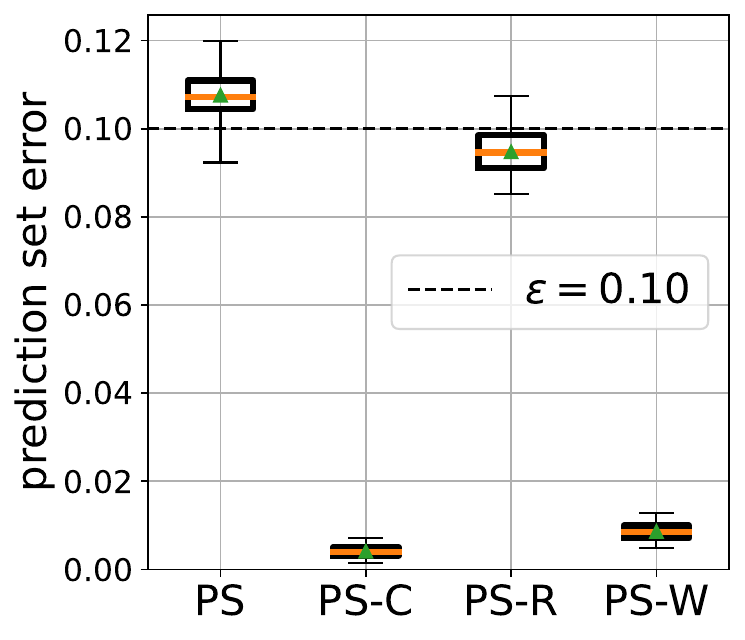}
\includegraphics[height=0.4\linewidth]{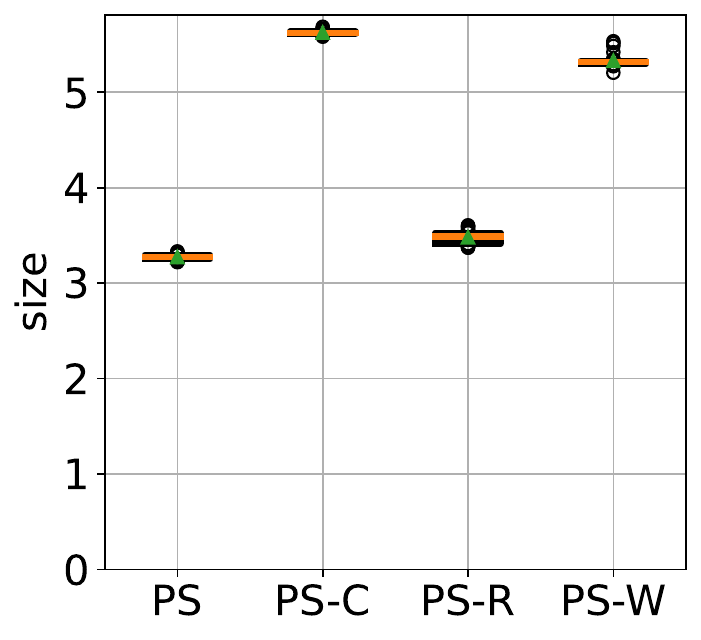}
\caption{Prediction set error and size on the ChestX-ray dataset. Parameters are $\epsilon = 0.1$, $\delta = 5\times 10^{-4}$, $m=67200$, $n=35200$, and $o=3520$.}
\end{subfigure}
\caption{More Prediction set results with different hyperparameters on the ChestX-ray dataset.}\label{fig:entity13}
\end{figure}

\subsection{Additional Baselines}\label{sec:morebase}
Class-conditional conformal predictors fit separate thresholds for each label and demonstrate robustness to label shift.
In Figure~\ref{fig:label-wise}, we show the class-conditional results for both conformal prediction tuned for average coverage and our PAC prediction set, on the CDC dataset.
Here, LWCP is a baseline from the label-conditional setting from 
\cite{sadinle2019least}, which does not satisfy a PAC guarantee. PS-LW adapts the standard PAC prediction set algorithm \cite{vovk2012conditional,park2020pac} 
to the label-conditional setting; our approach is PS-W. 
Most relevantly, while PS-LW approximately satisfies the desired error guarantee, it is more conservative than our approach (PS-W) and produces prediction sets that are larger on average. Intuitively, it satisfies a stronger guarantee than necessary
for our setting, leading it to be overly conservative.

\begin{figure}[h]
\centering\small
\includegraphics[height=0.35\linewidth]{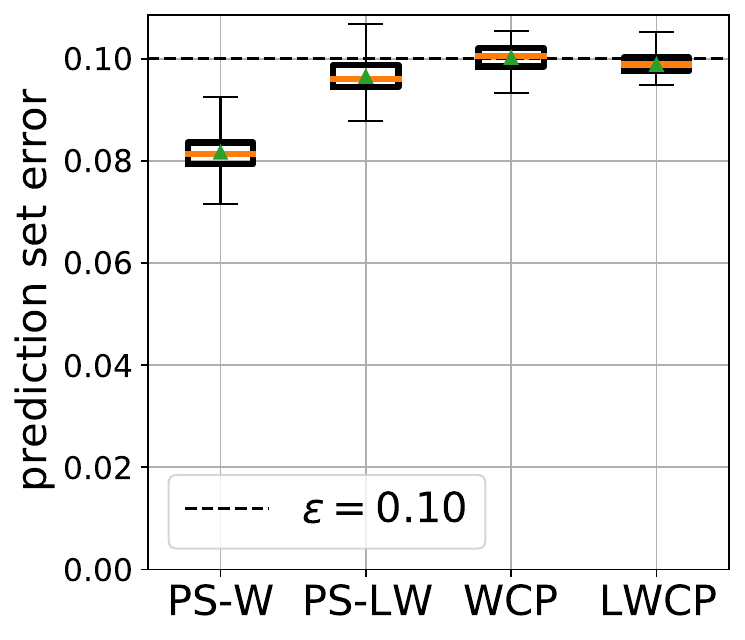}
\includegraphics[height=0.35\linewidth]{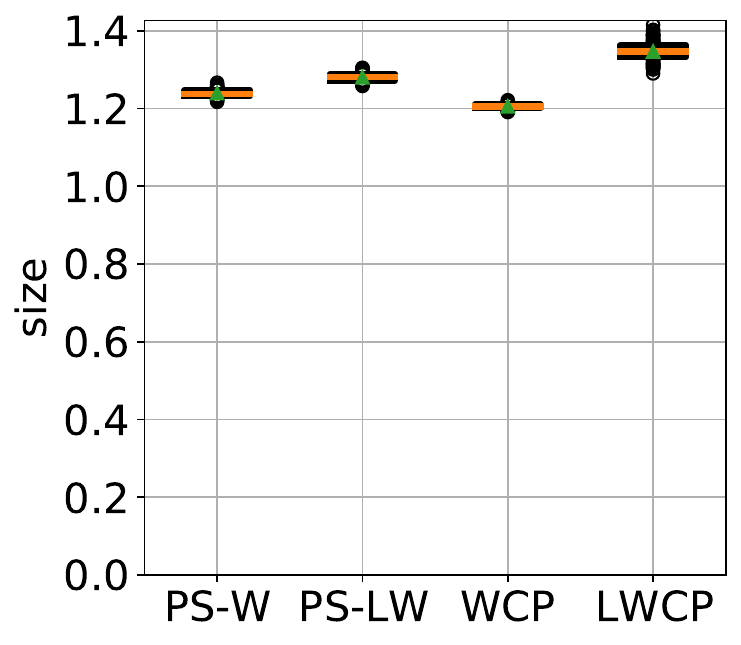}
\caption{Prediction set error and size under \textit{small} shifts on the CDC dataset. Parameters are $\epsilon = 0.1$, $\delta = 5\times 10^{-4}$, $m=42000$, $n=42000$, and $o= 9750$.}\label{fig:label-wise}
\end{figure}

Empirically, we find that while APS improves coverage in the label shift setting, it does not satisfy our desired PAC guarantee. In particular, we show results for the APS scoring function with vanilla prediction sets in Figure~\ref{fig:arc}; as can be seen, it does not satisfy the desired coverage guarantee. Due to its unusual structure, it is not clear how APS can be adapted to the PAC setting, which is our focus.

\begin{figure}[t]
\centering\small
\includegraphics[height=0.35\linewidth]{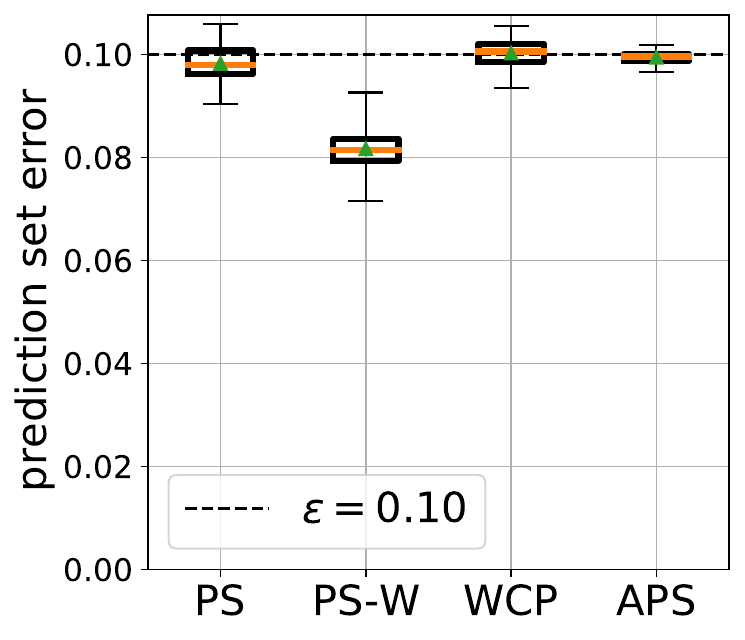}
\includegraphics[height=0.35\linewidth]{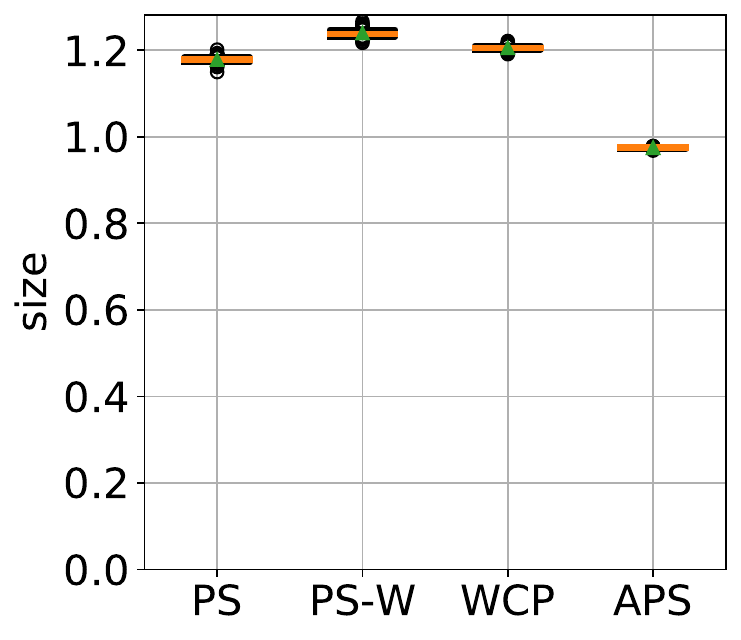}
\caption{Prediction set error and size under \textit{small} shifts on CDC dataset. Parameters are $\epsilon = 0.1$, $\delta = 5\times 10^{-4}$, $m=42000$, $n=42000$, and $o= 9750$.}\label{fig:arc}
\end{figure}

\end{document}

%% file: math_commands.tex
\usepackage{amsmath,amsfonts,bm}

\def\eqref#1{equation~\ref{#1}}

\def\1{\bm{1}}

\DeclareMathAlphabet{\mathsfit}{\encodingdefault}{\sfdefault}{m}{sl}
\SetMathAlphabet{\mathsfit}{bold}{\encodingdefault}{\sfdefault}{bx}{n}

